\documentclass{article}

\usepackage{microtype}
\usepackage{graphicx}
\usepackage{booktabs} 

\usepackage{hyperref}

\usepackage[accepted]{icml2024}

\usepackage{amsmath}
\usepackage{amssymb}
\usepackage{mathtools}
\usepackage{amsthm}
\usepackage[english]{babel}
\usepackage{blindtext}
\usepackage[hypcap=false]{subcaption}
\usepackage[hypcap=false]{caption}
\usepackage{multirow}
\usepackage{array}
\usepackage[11pt]{moresize}
\usepackage{xspace}
\usepackage{arydshln}
\usepackage{graphicx}
\usepackage{amsfonts}
\usepackage{booktabs}
\usepackage{gensymb}
\usepackage{color}
\usepackage{natbib}
\usepackage{hyperref}
\usepackage{soul}
\usepackage{makecell}
\usepackage{bm}
\usepackage{makecell}

\usepackage[capitalize,noabbrev]{cleveref}

\newcommand{\cs}[1]{{\leavevmode\color{black}{#1}}}
\newcommand{\hlr}[1] {{\color{red} #1}}

\newcommand{\PE}{\mathrm{PE}}
\newcommand{\APE}{\mathrm{APE}}
\newcommand{\SPE}{\mathrm{SPE}}
\newcommand{\vx}{\mathbf{x}}

\theoremstyle{plain}
\newtheorem{theorem}{Theorem}[section]

\theoremstyle{definition}

\theoremstyle{remark}
\newtheorem{remark}[theorem]{Remark}

\usepackage[textsize=tiny]{todonotes}
\usepackage[english]{babel}
\usepackage{blindtext}
\usepackage[hypcap=false]{subcaption}
\usepackage[hypcap=false]{caption}
\usepackage{multirow}
\usepackage{array}
\usepackage[11pt]{moresize}
\usepackage{xspace}
\usepackage{arydshln}
\usepackage{graphicx}
\usepackage{amsmath}
\usepackage{amsfonts}
\usepackage{booktabs}
\usepackage{gensymb}
\usepackage{color}
\usepackage{natbib}
\usepackage{hyperref}
\usepackage{mathtools}
\usepackage{soul}
\usepackage{makecell}
\usepackage{tabularx}
\usepackage{mathtools}
\usepackage{lipsum}

\icmltitlerunning{Learning High-Frequency Functions Made Easy with Sinusoidal Positional Encoding}

\begin{document}

\twocolumn[
\icmltitle{Learning High-Frequency Functions Made Easy \\ with Sinusoidal Positional Encoding}




\begin{icmlauthorlist}
\icmlauthor{Chuanhao Sun}{yyy}
\icmlauthor{Zhihang Yuan}{yyy}
\icmlauthor{Kai Xu}{comp}
\icmlauthor{Luo Mai}{yyy}
\icmlauthor{N. Siddharth}{yyy,aaa}
\icmlauthor{Shuo Chen}{yyy}
\icmlauthor{Mahesh K. Marina}{yyy}
\end{icmlauthorlist}

\icmlaffiliation{yyy}{The University of Edinburgh, Edinburgh, UK}
\icmlaffiliation{comp}{MIT-IBM Watson AI Lab, Cambridge, MA, US}
\icmlaffiliation{aaa}{The Alan Turing Institute, UK}

\icmlcorrespondingauthor{Chuanhao Sun}{chuanhao.sun@ed.ac.uk}
\icmlcorrespondingauthor{Mahesh K. Marina}{mahesh@ed.ac.uk}

\icmlkeywords{Machine Learning, ICML}

\vskip 0.3in
]



\printAffiliationsAndNotice{}  

\begin{abstract}
Fourier features based positional encoding (PE) is commonly used in machine learning tasks that involve learning high-frequency features from low-dimensional inputs, such as 3D view synthesis and time series regression with neural tangent kernels. Despite their effectiveness, existing PEs require manual, empirical adjustment of crucial hyperparameters, specifically the Fourier features, tailored to each unique task. Further, PEs face challenges in efficiently learning high-frequency functions, particularly in tasks with limited data. In this paper, we introduce sinusoidal PE (SPE), designed to efficiently learn adaptive frequency features closely aligned with the true underlying function. Our experiments demonstrate that SPE, without hyperparameter tuning,  consistently achieves enhanced fidelity and faster training across various tasks, including 3D view synthesis, Text-to-Speech generation, and 1D regression. SPE is implemented as a direct replacement for existing PEs. Its plug-and-play nature lets numerous tasks easily adopt and benefit from SPE. 

\noindent \textbf{Code}: \url{github.com/zhyuan11/SPE}
\end{abstract}
\vspace{-2em}

\section{Introduction}\label{intro}
Fully connected neural networks, a.k.a~multilayer perceptrons or MLPs, are trained to generate representations of high-dimensional data such as shapes, images, and signed distances, by processing low-dimensional coordinates.
Recent works have shown that the Fourier series regression using MLP~\citep{tancik2020fourier} can enable neural networks to learn high-frequency functions in low-dimensional spaces. In neural radiance fields (NeRFs) and its follow-up studies~\citep{tancik2020fourier,mildenhall2021nerf}, Fourier features, induced by positional encoding (PE), are applied to learn 3D representations of objects or scenes by taking in 1D sequences that represent samples of light.

Despite its effectiveness, a successful application of PE is non-trivial. 
A series of studies emerge to resolve the practical issues of PE on its sensitivity to hyper-parameters~\citep{gao2023adaptive}, difficulty to capture high-frequency components during training~\citep{yang2023freenerf}, etc. Those practical issues also block a wider application of PE in emerging generative AI tasks, especially those involving complex high-frequency details. 
For example, our experiment shows that a direct application of PE in speech synthesis does not offer any benefit in capturing high-frequency details~\citep{ren2019fastspeech, ren2020fastspeech}, though ``by design'' it should do.

In this study, we delve into the challenges of training neural networks for machine learning tasks that demand the retention of high-frequency components in their outputs. 
Our exploration into the quantity and training dynamics of frequency components within PE has led us to identify two primary factors that can undermine the effectiveness of PE: (1) the difficulty in configuring stationary frequency features without adequate prior knowledge, which can result in PE failing to learn if configurations are incorrect, and (2) the detrimental impact on performance caused by compelling the model to perfectly align with the specific frequency components of the training set (e.g., the case of overfitting PE for an original NeRF).

To overcome the challenges identified, we seek to develop a new PE that can effectively learn the appropriate number of components and their frequencies \emph{conditionally on the inputs}. Our development has led to sinusoidal positional encoding (SPE) which augments PE \citep{tancik2020fourier} with periodic activation functions~\citep{sitzmann2020implicit}, thus making the number and frequencies of Fourier series trainable and adaptive to inputs.
Although we initially propose SPE for challenging few-view NeRF tasks, we find it is a generic method that can benefit a diverse set of tasks which need modelling complex high-frequency features.
With SPE, we have a simple yet effective form of PE for the first time that can effectively work in a wide range of tasks without manually tuning the numbers as well as the values of the Fourier features. 

We have evaluated SPE against a wide range of baseline methods in various generation tasks. In the task of few-view NeRF, we achieve a significant gain in synthesizing high-frequency details with limited views by replacing PE with SPE. In the task of text-to-speech generation, we achieve a significant performance gain on a state-of-the-art (SOTA) model: FastSpeech~\citep{ren2019fastspeech} with \emph{a single-line change} of codes. In the task of 1D regression with neural tangent kernel (NTK) \citep{jacot2018neural}, we significantly enhance both the fidelity and convergence speed by simply replacing the PE with SPE, which also avoids the need for expensive hyperparameter tuning.



\section{Background and Motivation}\label{statement}
PE is designed for learning high-frequency functions in machine learning tasks such as NeRF~\citep{tancik2020fourier, mildenhall2021nerf}, 1D or 2D regression~\citep{tancik2020fourier, nguyen2015deep}, 3D shape regression~\citep{mescheder2019occupancy} and audio generation~\citep{ren2019fastspeech, ren2020fastspeech}. 
It refers to the expansion of input $\mathbf{x}$ with sinusoidal function pairs $(\sin(\alpha_i\mathbf{x}),\cos(\alpha_i\mathbf{x}))$, $i\in \{0,\dots ,L-1\}$ where $L$ is a hyper-parameter determining the number of frequency components, also known as a special case Fourier features in \citep{rahimi2007random}.
In PE, the Fourier features $\alpha_i$ are predefined per scenario empirically. 
A typical default setup in NeRF-related tasks~\citep{mildenhall2021nerf, yang2023freenerf, tancik2023nerfstudio} with $\alpha_i = 2^{i-1}\pi$ is illustrated in Equation~\ref{eq:nerf-pe}:
\begin{equation}\label{eq:nerf-pe}
\begin{aligned}
    \mathrm{PE}_L(\mathbf{x}) =
    [
    \sin(\pi\mathbf{x}),
    &\cos(\pi\mathbf{x}),
    \cdots,\\
    &\sin(2^{L-1}\pi\mathbf{x}),
    \cos(2^{L-1}\pi\mathbf{x})
    ]^\top
\end{aligned}
\end{equation}

However, such a formulation has obvious limitations.
The optimal setting of $L$ should be conditional on two factors: (i) the task setup and (ii) the number of parameters in an MLP. While the second factor is a natural concern, the task-specific setup makes an optimal $L$ difficult to set. We illustrate how the optimal setting of $L$ will affect the effectiveness of a learnt high-frequency function in the context of a NeRF task in Figure~\ref{fig:pecompare}. 
\begin{figure}[t]
\centering
\includegraphics[width=0.49\textwidth]{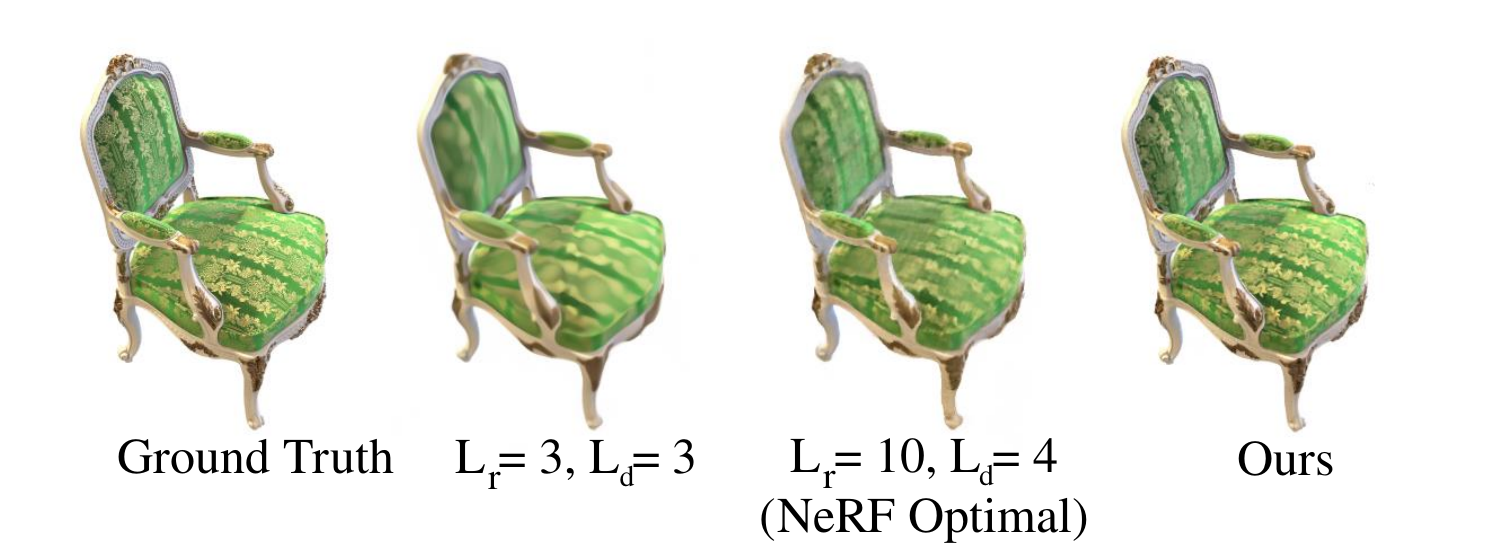} 
\vspace{-1em}
\caption{\cs{New view generation in NeRF with 8 input views on Blender dataset~\citep{mildenhall2021nerf}. $L_r$ is the number of components taken when processing coordinates in PE and $L_d$ for the direction processing in PE.}\looseness=-1}
\label{fig:pecompare}
\vspace{-1em}
\end{figure}
We find that when setting $L$ to 10 and 4 for RGB and density threads respectively, it achieves significantly better performance than using $L=3$ for both threads.
Nevertheless, our method, SPE, can improve upon it, as shown in the last column of the same figure. 


\cs{In \citep{rahimi2007random, tancik2020fourier}, another variation of PE is discussed as well (called Gaussian Random Fourier Features or GRFF), where the Fourier features $\alpha_i$ is defined as a pseudo random sequence sampled from a Gaussian distribution $\mathcal{N}(0,\sigma^2)$, with a form
\begin{equation}\label{eq:nerf-grff}
\begin{aligned}
    \mathrm{GRFF}_L(\mathbf{x}) =
    [
    \sin(\mathbf{B}\mathbf{x}),
    \cos(\mathbf{B}\mathbf{x})
    ]^\top
\end{aligned},
\end{equation}

where $\mathbf{B}\in \mathbb{R}^{L\times d}$, $d$ is the input sequence length. Although sample from a Gaussian distribution, the GRFF still uses a stationary encoding methodology. In practice, GRFF only brings negligible gain on tasks such as NeRF, and so far PE is the common option. For the clarity and convenience of analysis, we focus the discussion with PE in this paper, and the same conclusion can be applied to GRFF as well. More detailed evaluation and discussion about GRFF can be found in \S\ref{eval_grff}.}


To optimize the hyper-parameters of PE, existing approaches follow two ways: (i) empirically tune the parameters of stationary Fourier features for PE, or (ii) make the parameters associated with Fourier features trainable.

\subsection{Empirically optimized stationary Fourier features}
\cs{
NeRF, our major use case, is a technology to generate new views based on existing views of the same object or scene. Few 2D images representing available views along with the 3D coordinates of image pixels and direction of viewpoint make up the input to the network. The network is then trained to generate the density (how much light is blocked or absorbed at that point) and RGB colors from new input viewpoints. Basically, a MLP is used in NeRF with two different input threads: 3D coordinates and direction of viewpoint, where the two threads are processed with different feature length $L$ of PE for better representing high frequency details. The choice of $L$ significantly influences the performance of NeRF.
}

Practitioners often empirically choose stationary $L$ based on the specific task setup. 
We show this empirical approach is problematic using concrete examples. 
Specifically, we first assess a few-view NeRF model~\citep{mildenhall2021nerf} with varied $L$ for different objects on the Blender dataset (used in~\citep{sitzmann2019deepvoxels, mildenhall2021nerf}). 
Additionally, we repeat the same experiments for a SOTA speech synthesis model, named FastSpeech~\citep{ren2019fastspeech}.

For the few-view NeRF , we summarize the results in Table~\ref{table:NeRFTest}. 
For different objects to generate, different settings for $L$ bring about $2\text{dB}$ variation on Peak Signal-to-Noise Ratio (PSNR). 
\begin{table}[t]
\centering
\resizebox{\columnwidth}{!}{%
\setlength{\tabcolsep}{4pt}
\begin{small}
\begin{tabular}{lccccccccc}
\toprule
 PSNR$\uparrow$&Chair  &Drums  &Ficus  &Hotdog  &Lego  &Materials  &Mic  & Ship & Average  \\ \midrule
L=5 &32.19  &\textbf{\textit{25.29}}  &\textbf{\textit{30.73}}  &36.06  &30.77  &\textbf{\textit{29.77}}  &31.66  &28.26 &30.59 \\ 
L=10 &\textbf{\textit{33.00}}  &25.01  &30.13  &\textbf{\textit{36.18}}  & \textbf{\textit{32.54}}  &29.62  &\textbf{\textit{32.91}}  &\textbf{\textit{28.65}}  & \textbf{\textit{\hlr{31.00}}} \\ 
L=15 &32.87  & 24.65 &29.92  &35.78  &32.50  &29.54  &32.86  &28.34 &30.80\\ \hline 
Best &33.0  & 25.29 &30.73  &36.18  &32.54  &29.77  &32.91  &28.65 &31.13\\ \bottomrule
\end{tabular}
\end{small}
\setlength{\tabcolsep}{6pt}
}
\caption{NeRF with different settings for $L$ in PE.}
\label{table:NeRFTest}
\vspace{-1em}
\end{table}
However, none of the settings for $L$ can achieve the best performance on all objects. Overall, we find that $L=10$ yields the best average performance on this task.

We apply the best setting of $L=10$ from above to another task: Fastspeech.
We apply PE to the MLP layer at the end of the text-to-speech transformer. 
Here we would expect this modification with PE can bring more details since, if we use MLP without PE, the output audio loses many details with a blurring spectrum.
According to Figure~\ref{fig:pespeech}, with all $L\in \{5, 10, 15\}$, adding PE in form as Equation~\ref{eq:nerf-pe} does not change the result, supporting our claim that the stationary $L$ which shows best performance in a task cannot be applied to other tasks. 
\begin{figure}[t]
\centering
\includegraphics[width=0.45\textwidth]{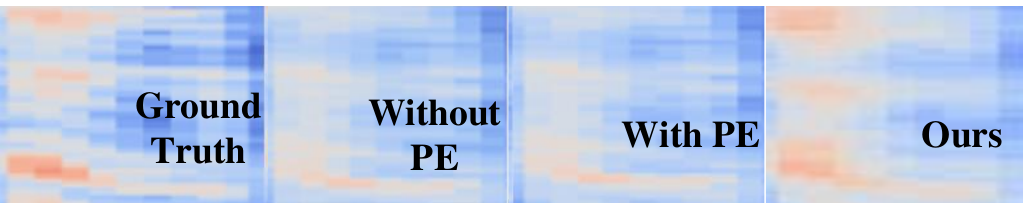} 
\vspace{-1em}
\caption{The Optimal PE for NeRF on Blender dataset only has negligible influence on speech generation with FastSpeech, while our method achieves better alignment of
the red regions with the ground truth.}
\label{fig:pespeech}
\vspace{-1em}
\end{figure}

It is worth noting that for both NeRF and speech synthesis tasks, SPE achieves the best performance, even better than all methods that rely on extensive manual tuning.


\subsection{Adaptive Fourier features \& probabilistic encoding}


To address the issues with stationary Fourier features, adaptive positional encoding~\citep{gao2023adaptive} has been proposed to make the Fourier features trainable. More formally, a trainable Fourier feature $\APE(\cdot)$ is defined as
\begin{equation}\label{eq:ape}
\begin{aligned}
    \APE_{\boldsymbol{\mathbf{W}_{\text{SPE}}}}(\mathbf{x}) = 
    [
    \sin(\omega_0\mathbf{x}),
    &\cos(\omega_0\mathbf{x}),
    \cdots,\\
    &\sin(\omega_K\mathbf{x}),
    \cos(\omega_K\mathbf{x})\
    ]^\top
    \end{aligned}
\end{equation}
where $\mathbf{W}_{\text{SPE}} = [\omega_1, \dots, \omega_K]$ is the trainable frequency features and $K$ is the number of possible features. 
However, training $\omega$ from data is still challenging. For an MLP, the input $\mathbf{x}$ is a 1D sequence, whereas the output matrix $[\omega_i]$  of $\omega$ is a $N\times K$ matrix, $N$ is the length of $\mathbf{x}$. The learning task for training $[\omega_i]$ is a non-trivial 1D-to-2D problem that requires learning high dimensional functions in unbounded space (\S\ref{sec:altlearn}), complicating the training of $[\omega_i]$. From Figure \ref{fig:instantnerf}, we can observe that while APENeRF can produce a reasonable outline of the chair, the detailed patterns on the chair appear blurred. In contrast, our method (i.e., SPE) mitigates this blurriness, generating more accurate flower patterns.


Apart from APE, probabilistic encoding methods, such as hash encoding~\citep{muller2019neural, muller2022instant, tancik2023nerfstudio} are also proposed to learn Fourier features based on inputs.
These methods, though adaptive, exhibit a high dependency on the amount of available training data. Without sufficient data, an effective encoding is often difficult to find (which is detailed in \S\ref{hashreasoning} and experimentally shown in \S\ref{subsec:nerf}).
Essentially, the encoding part of these works performs a Monte Carlo search and the rest of the network need to take care of learning the composition of those encodings and handle Hash collision, which requires many more samples to capture sufficient statistics. If we train a NeRF with hash encoding with a limited number of views (denoted as InstantNeRF---the SOTA NeRF which adopts hash encoding), the fidelity is poor, as Figure~\ref{fig:instantnerf} shows.
 In contrast, SPE generates a high-quality object close to the ground truth, achieving noticeably enhanced performance over InstantNeRF.

\begin{figure}[t]
\centering
\includegraphics[width=0.49\textwidth]{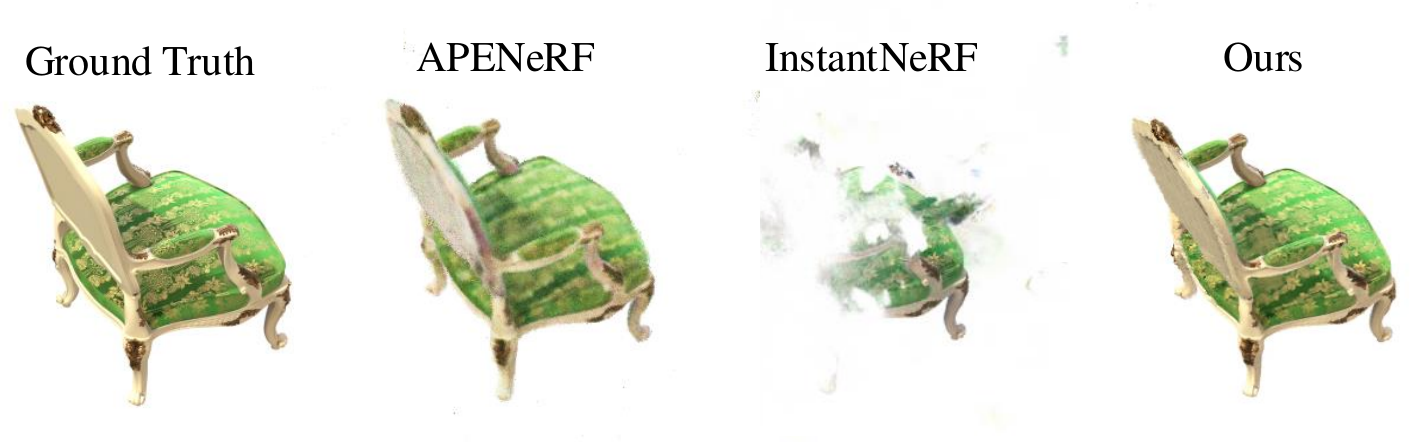} 
\vspace{-1em}
\caption{Objects generated by APENeRF, InstantNeRF and our method. APENeRF uses hash encoding and it is hard to train with 8 views on the Blender synthetic dataset whereas ours (i.e., SPE), even with limited data, can already achieve high-quality generation close to the ground truth.}
\label{fig:instantnerf}
\vspace{-1em}
\end{figure}


\section{Sinusoidal Positional Encoding}
We want to make the adoption of PE in generative AI tasks simple and effective. To achieve this, we design \emph{Sinusoidal Positional Encoding} or SPE, defined as follows
\begin{equation}\label{eq:spe}
    \begin{aligned}
        \SPE(\vx)
        &= \sin(\boldsymbol{\omega}\PE(\vx)) \\
    \end{aligned}
\end{equation}
where the $\boldsymbol{\omega}$ is a trainable vector that represents the learned features, $L \in \mathbb{N}^+$; the value of $\mathrm{PE}(\mathbf{x})$ is between $[-1,1]$. With MLP, SPE has trainable matrix $\mathbf{W}_\text{SPE}$ in the form as:
\begin{equation}\label{speeq2}\nonumber
    \begin{aligned}
        &\sin(\mathbf{W}_\text{SPE} \mathrm{PE}(\mathbf{x})) =\\
        &\qquad[
        \sin(\boldsymbol{\omega}_1\sin(\pi\mathbf{x})), \sin(\boldsymbol{\omega}_2\cos(\pi\mathbf{x})), \dots, \\        &\qquad\phantom{[}\sin(\boldsymbol{\omega}_{2L-1}\sin(2^{L-1}\pi\mathbf{x})), \sin(\boldsymbol{\omega}_{2L}\cos(2^{L-1}\pi\mathbf{x}))
        ]^\top
    \end{aligned}
\end{equation}
where $\mathbf{W}_\text{SPE} = [\boldsymbol{\omega}_1, \dots, \boldsymbol{\omega}_{2L}]$ is the weight of the first \cs{fully connected layer}, $L$ is the number of components after PE, the function $\sin(\cdot)$ is performed by a sinusoidal activation.
Note that we omit the bias term to simplify the discussion.

Intuitively, we observe that the behavior of representation in \eqref{eq:spe} is close to a normal sinusoidal wave with trainable features. A comparison between SPE learned features and the hard-coded features in PE is illustrated in Figure~\ref{fig:learnedfeature}, where we use larger $L=12$ to search the features. We find that, with SPE, even with large $L$ the actual learned frequency is within the band where $L=10$, which explains why PE takes $L=10$ as the optimal setup. From  Figure~\ref{fig:learnedfeature}, we also notice that the difference between objects on features is mainly high-frequency components. Therefore, learning high-frequency details are critical to ensure a high fidelity synthesis.\looseness=-1

We also find that the proposed form in Equation \eqref{eq:spe} is significantly more efficient to train than learning Fourier features which directly operate on $\mathbf{x}$, i.e.~learning $\sin(\omega\mathbf{x})$ as in Equation~\eqref{eq:ape} \citep{gao2023adaptive}. 
We analyze this enhanced learning efficiency with detail in \S\ref{sec:altlearn}.

\begin{figure}[t]
\centering
\includegraphics[width=0.48\textwidth]{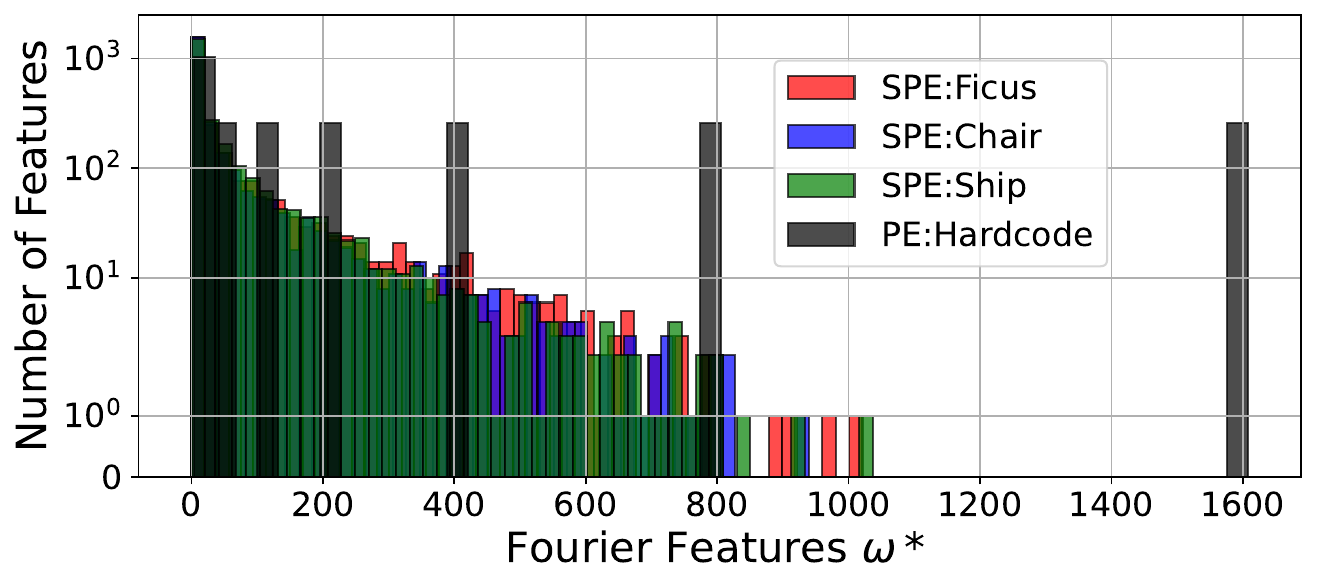} 
\vspace{-1em}
\caption{Learned features by SPE in different objects in the Blender dataset. A learned feature rarely goes beyond $L=9$, and therefore set $L=10$ is the optimal configuration for PE. $\omega^*$ is the feature: $\omega^* = \omega\cdot2^{l-1}, l\in\{1, 2, \dots, L\}$.}
\label{fig:learnedfeature}
\vspace{-1em}
\end{figure}

\subsection{Design Choices}
We discuss different possible designs for Equation~\eqref{eq:spe}. We could consider (1) other periodic activation functions instead of the sinusoidal function and (2) other encoding methods instead of sinusoidal-based encoding. 

\paragraph{Choice of periodic activation functions} We first discuss which periodic activation functions are effective for SPE. According to the Fourier theorem, any continuous series can be expanded as a combination of sinusoidal waves. Therefore representing sinusoidal waves is critical to guarantee that all target outputs can be potentially represented. 
Basically, to make sure the network can represent the sinusoidal wave, Equation \ref{approxtarget0} in Theorem~\ref{thm:trainable-spe} must be met, otherwise the rest of the network after SPE has to learn a weight that is conditional on the Fourier features, which makes the learning more difficult. For instance, if there is significant linear mapping in the activation, then $\mathbf{I}(\cdot)$ or $\mathbf{S}(\cdot)$ will have a form (proof in \S\ref{otheract}) where each feature has to be learned in coupling with specific input $\mathbf{x}$ as: \begin{equation}\label{eqst_0}
    \mathbf{S}(t) = \frac{\sin(\boldsymbol{\omega}\cdot t)}{\sqrt{\boldsymbol{\omega}^2-(2n\pi)^2} \bmod 2\pi}
\end{equation}
where $\boldsymbol{\omega}$ is a trainable vector and $t = 2^{l}\pi\mathbf{x}, l\in\{0,\dots, L-1\}$. The function in Equation~\ref{eqst_0} is much harder to learn than Equation~\ref{IS_0} and \ref{IS_1}, potentially representing wrong frequency features for $\mathbf{x}$ that is not in the training set. Therefore we propose using sinusoidal activation as a simple-yet-effective periodic activation. \cs{As for the other common periodic activation functions, we empirically test their performance in NeRF and find that they perform significantly worse than the sinusoidal function. More detailed discussion about other activation functions in \S\ref{eval_other_func}.}

\paragraph{Choice of encoding methods}\label{hashreasoning}

We now discuss different candidate encoding methods: one blob encoding and hash encoding~\citep{muller2022instant}. While claiming faster training, we find that the hash encoding based method shows much worse performance than the PE or SPE-based method in few-view NeRF.
Instead of using sinusoidal functions, in hash encoding the input $\mathbf{x}$ is encoded by hashing~\citep{muller2022instant}
\begin{equation} \label{eq:hash-encoding}
   h(\mathbf{x}) =  \left(\bigoplus^d_{i=1}x_i\pi_i\right) \mod T
\end{equation}
where $\oplus$ denotes the bit-wise XOR operation and $\pi_i$ are unique,
large prime numbers. Looking from the PE viewpoint, hash encoding uses a pseudo random function~\citep{muller2019neural} to replace the sinusoidal function of PE. Essentially, since there are multiple hash tables, hash encoding essentially performs a Monte Carlo search for proper encoding. Then the model needs to learn from the encoding to match actual features. 

When using hash encoding, there will be a hash collision, where different $x_i$ might be mapped into the same value in a pseudo-random manner. In the implementation with MLP, such as InstantNeRF~\citep{muller2022instant} and NeRFfacto~\citep{tancik2023nerfstudio}, they do not explicitly handle collisions of the hash functions by typical means like probing, bucketing, or chaining. Instead, we rely on the neural network to
learn to disambiguate hash collisions themselves, avoiding control flow divergence, reducing implementation complexity and improving performance. 
Through the illustration from Figure \ref{fig:instantnerf} and an experiment in \S\ref{subsec:nerf}, we find that NeRF with hash encoding on a small number of training views performs much worse than SPE in terms of fidelity. Both searching for appropriate encoding and learning to fix hash collision requires lots of effective inputs (\emph{i.e.}, distinct views in NeRF), and hence it is hard to train NeRF with Hash encoding on limited views. 

Generally speaking, although hash encoding has a significant advantage in terms of computation (training time), its learning efficiency, i.e.~the ability to extract appropriate features, is worse than SPE.

\subsection{Practical implementation}\label{sec:implementation}
\begin{figure}[t]
\centering
\includegraphics[width=0.45\textwidth]{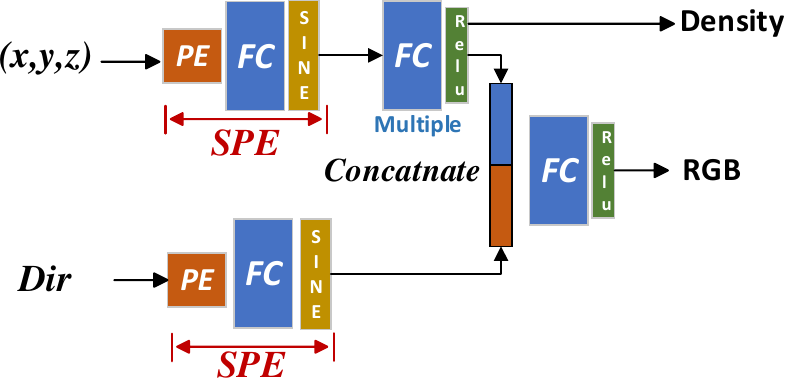} 
\vspace{-1em}
\caption{Example of implementing SPE in NeRF: Using periodic activation function for Frequency Encoded series. (x,y,z) is the coordinates of the object. Dir indicates direction of the view.}
\label{fig:encoding}
\vspace{-1em}
\end{figure}

Here we provide a practical guide to implement SPE. For a given network, one can replace the first activation function of the MLP after PE with a sinusoidal activation. We illustrate this using an SPE-enabled NeRF in Figure~\ref{fig:encoding}. Incorporating SPE into existing neural networks requires no additional hyperparameter tuning and maintains a similar network structure, thereby preserving the original training process for seamless integration.

The only hyper-parameter of SPE is the $L$. 
Since Fourier feature $\omega$ is trainable, larger the value of $L$ in SPE the better and should match the network size, i.e., $L \propto N_{\text{para}}$ where $N_{\text{para}}$ is the number of parameters in the network, which we establish in Theorem~\ref{thm:pe-accuracy}.
\begin{theorem}\label{thm:pe-accuracy}
    $L$ determines the approximation accuracy of SPE to a trainable PE (proof in \S\ref{spelearn}).
    \begin{equation}
    \sin(\boldsymbol{\omega}\mathrm{PE}(\mathbf{x})) = \PE(\boldsymbol{\omega}(\mathbf{x}+\frac{1}{2^L}))
    \end{equation}
\end{theorem}
Using larger $L$ should not cause unexpected artefacts because the features in the PE part will be tuned by $\boldsymbol{\omega}$. In our experiments, we can optimize $L$ by increasing the number until the result does not improve.

\subsection{Analysis of representation power}

In the following, we analyze the representation power of SPE by showing that SPE after learning can effectively represent PE with optimal parameters. 

\begin{theorem}\label{thm:trainable-spe}
For arbitrary input $\mathbf{x}$, the network can learn $\omega$ agnostic functions $\mathbf{I(\cdot)}$ and $\mathbf{S(\cdot)}$ to make SPE have the same effect as PE with Fourier features tuned by $\omega$.
\end{theorem}
\begin{proof}
Let $t = 2^{l}\pi\mathbf{x}, l\in\{0,\dots, L-1\}$
\begin{align}\label{approxtarget0}\nonumber
    &\exists \mathbf{I}(t), \mathbf{S}(t) \to \\ & \sin(\mathrm{\boldsymbol{\omega}}\sin(t))\cdot\mathbf{I}(t) + \sin(\boldsymbol{\omega}\cos(t))\cdot \mathbf{S}(t) = \sin(\boldsymbol{\omega}\cdot t)
\end{align}
As we prove in \S\ref{spelearn}, $\mathbf{I(\cdot)}$ and $\mathbf{S(\cdot)}$ have a simple form
\begin{equation}\label{IS_0}
    \text{If:}~t\rightarrow (n+\frac{1}{2})\pi,~\text{Then:}~\mathbf{I}(t) = 1, \mathbf{S}(t) = 0 
\end{equation}

\begin{equation}\label{IS_1}
    \text{If:}~t\rightarrow n\pi,~\text{Then:}~\mathbf{I}(t) = 0, \mathbf{S}(t) = 1
\end{equation}
Combining \ref{approxtarget0}, \ref{IS_0}, and \ref{IS_1}, the sinusoidal wave is approximated with both the $\sin$ and $\cos$ parts. 
\end{proof}
\begin{remark}
Intuitively, Equation \ref{IS_0} and \ref{IS_1} approach sinusoidal wave for different values of $\mathbf{x}$ by using different frequency parts of SPE to make the approximation. The $\mathbf{I(\cdot)}$ and $\mathbf{S(\cdot)}$ is $\omega$ agnostic and a simple linear binary classification of input value $\mathbf{x}$, therefore there is no extra overhead to learn such functions.
\end{remark}

\begin{theorem}\label{thm:pe-approximation}
Equation \ref{eq:spe} is an effective approximation to PE when the absolute value of learned feature $\omega$ is small
\begin{equation}\label{speapprox}
    \lim_{|\omega|\rightarrow 0} \sin(\omega\mathrm{PE}(\mathbf{x})) = \mathrm{PE}(\mathbf{x})
\end{equation}
\end{theorem}
Details proof in \S\ref{spehaspe}. In the case original PE has captured the appropriate features, SPE can easily approach PE by using small weights. The smaller weights can be scaled up by subsequent fully connected layers without impacting the corresponding frequency band's power. Therefore, we prove that PE is an easy-to-learn sub-case of SPE.

\subsection{Discussion of training efficiency}\label{sec:altlearn}

We discuss the choice of optimization methods together with SPE.
One option is to learn $\sin(\omega\mathbf{x})$ directly as in Equation~\eqref{eq:ape}. Compared with learning in the form of $\sin(\omega\mathbf{x})$, incorporating PE with SPE simplifies training by limiting the search to a bounded space, unlike other methods that operate in a larger, unbounded space. This constraint leads to more stable training and faster convergence.
Suppose we have $L$ components in SPE, then the hardcoded features are $\mathbf{f}_{h} =[\pi, 2\pi, \dots, 2^{L-1}\pi]$, and the learned features (suppose the actual features is distinct to the hardcoded features) $\mathbf{f}_{l} =[\omega_0\cdot\pi, \omega_1\cdot2\pi, \dots, \omega_L\cdot2^{L-1}\pi]$. Assume that the actual feature is $\omega_a$, the difference between the nearest hardcoded feature and actual feature $\Delta_{\text{PE}}$ would be
\begin{align}\label{pebound}
    \Delta_{\text{PE}} = \begin{cases}
\min(2^{\beta_0}\pi - \omega_a,\omega_a - 2^{\beta_1}\pi) & \omega_a < 2^{L-1}\pi\\
 \omega_a - 2^{L-1}\pi & \omega_a \geq 2^{L-1}\pi
\end{cases}
\end{align}

where $\beta_0 = \lceil\log_2\frac{\omega_a}{\pi}\rceil$ and $\beta_1 = \lfloor\log_2\frac{\omega_a}{\pi}\rfloor$. Therefore if we have PE, then tuning the frequency always can be converted to a bounded range that is significantly smaller than $\omega_a$ when $L$ is large enough. Without embedding PE in SPE, the searching space to train $\omega$ is not bounded and will be sensitive to initialization, which is hard to carry out when the actual feature is unknown.  

\cs{
\subsection{Quantifying Learned Fourier Features}\label{sec:quantfreq}
As prior works, we use standard metrics for NeRF including PSNR and SSIM. However, we find that PSNR and SSIM cannot show the fidelity of learnt Fourier features explicitly. Therefore, we further design metrics that can quantify the effects of high-frequency and low-frequency details. These new metrics include: 
\begin{enumerate}
    \item \emph{Wavelet Decomposition Power Ratio} (WDPR) measures the high-frequency band that contributes to gain in different methods. We first conduct wavelet decomposition with $\lambda$ levels on the image to separate the high-frequency details. Level-$\lambda$ wavelet power radio is defined as the power of signal at the $\lambda th$ level decomposition compared to the ground truth decomposition.
    The WDPR is then can be calculated as follow
\begin{equation}\label{waveletpower}
    \hspace*{-1.75em}\mathrm{WDPR}(\mathbf{y}_{\text{true}}, \mathbf{y}_{\text{syn}},\lambda) = \frac{|\mathbf{P}(W(\mathbf{y}_{\text{true}},\lambda)) - \mathbf{P}(W(\mathbf{y}_{\text{syn}},\lambda))|}{\mathbf{P}(W(\mathbf{y}_{\text{true}},\lambda))}
\end{equation}
where $\mathbf{P}(\mathbf{y}) = \sum y_i^2, y_i \in \mathbf{y}$, $\mathbf{y}_{\text{syn}}$ denotes synthesis view, and $\mathbf{y}_{\text{true}}$ denotes the ground truth synthesis view.
    \item \emph{Relative Wasserstein Distance Error} (RWDE) which assesses how accurately the model learns the distribution change. Let $y_{\text{train}}$ denote the training view, $y_{\text{syn}}$ denote a synthesis view, and $y_{\text{true}}$ denote the ground-truth synthesis view, the RWDE is defined as:
\begin{equation}
\mathrm{RWDE}(\mathbf{y}_{\text{train}}, \mathbf{y}_{\text{syn}}, \mathbf{y}_{\text{true}}) = \frac{\mathrm{WD}(\mathbf{y}_{\text{train}}, \mathbf{y}_{\text{syn}})}{\mathrm{WD}(\mathbf{y}_{\text{true}}, \mathbf{y}_{\text{syn}})}
\end{equation}
where $\mathrm{WD}(\mathbf{a},\mathbf{b})$ compute the Wasserstein distance between two images $\mathbf{a},\mathbf{b}$, and take the average of different RGB channels if needed.
Intuitively, RWDE gives a sense of whether the model tends to overfit the distribution of existing views. If $\text{RWDE}<1$ then it tends to overfit to existing views, and vice versa. 
\end{enumerate}

}

\section{Experiments}\label{sec:exp}

We evaluate SPE against different baseline methods using various generation tasks: few-view NeRF, Text-to-Speech synthesis, and 1D regression with NTK. 
For NeRF, we focus on the case with few views and compare the fidelity of the new view synthesis (\S\ref{subsec:nerf}). For the speech synthesis task, we focus on improving the FastSpeech~\citep{ren2019fastspeech} method, where the insufficient accuracy of MLP comes as a main obstacle for high-quality generation.
We show the use of SPE significantly improves their performance with a single-line of change in their open-source code base (\S\ref{sec:speech-synthesis}).
Finally, we compare SPE against PE in 1D regression with NTK, following its original paper~\citep{tancik2020fourier}, for which we observe that SPE shows a significant gain in convergence speed and fidelity by simply changing one activation function (\S\ref{sec:ntk}).

\subsection{Few-View NeRF}\label{subsec:nerf}

We implement SPE in NeRF following the way depicted in Figure~\ref{fig:encoding}. By default, we build our SPE method on top of the SOTA model: FreeNeRF~\citep{yang2023freenerf}. Specifically, for the training with SPE on FreeNeRF, we find that it is effective to train FreeNeRF and SPE with adversarial loss to minimize the Wasserstein distance to the target view. Details of the experimental setup can be found in \S\ref{detailedconfig}, and more analysis results can be found in \S\ref{appendix:EMD}.

For baselines, we  consider PE-based methods, including DietNeRF~\citep{jain2021putting}, MipNeRF~\citep{barron2021mip}, and FreeNeRF~\citep{yang2023freenerf}, as well as those with hash encoding, such as InstantNeRF~\citep{muller2022instant} and NeRFfacto~\citep{tancik2023nerfstudio}. APENeRF~\citep{gao2023adaptive} is as far as we know the only method that learns the Fourier features explicitly, so we also include this work as a baseline. All the baselines are open-source and we implement SPE on top of their official implementation.

\begin{table*}
\centering
\setlength{\tabcolsep}{4pt}
\begin{small}
    \begin{tabular}{l|cccc|cc|c|ccc}
    \toprule
    \multicolumn{1}{c|}{}&\multicolumn{4}{c|}{Non-Adaptive PE}&\multicolumn{2}{c|}{Hash Encoding}&\multicolumn{1}{c|}{APE}&\multicolumn{3}{c}{SPE}\\ \hline
    Metric&NeRF  &DietNeRF  &MipNeRF  &FreeNeRF  &InstantNeRF  &NeRFfacto  &APENeRF &\makecell{MipNeRF\\+SPE} & \makecell{FreeNeRF\\+SPE}  & \makecell{FreeNeRF\\+SPE+GAN}\\
    \midrule
    PSNR $\uparrow$ &14.983  &23.142  &23.344  &24.259  & 14.681  &14.934  &23.067  &23.563  &24.951  &25.202 \\
    SSIM $\uparrow$ &0.689  & 0.865 &0.879  &0.883  &0.676  &0.682  &0.863  &0.885  &0.898  &0.910 \\
    \bottomrule
    \end{tabular}%
    \end{small}
\caption{Performance comparison of NeRF models and encoding methods.}
\label{table:performance-comparison-models}
\vspace{-1em}
\end{table*}

\begin{table}
\centering
\setlength{\tabcolsep}{4pt}
\begin{small}
    \begin{tabular}{lcccc}
    \toprule
    Metric & FreeNeRF  & w/ SPE  & w/ GAN  & w/ SPE + GAN  \\
    \midrule 
    PSNR $\uparrow$ & 24.259 & 24.951 & 24.531 & 25.202 \\
    SSIM $\uparrow$ & 0.883 & 0.898 & 0.889 & 0.910 \\
    \bottomrule
    \end{tabular}%
    \end{small}
\setlength{\tabcolsep}{6pt}
\caption{Ablation Study on FreeNeRF.}
\label{table:ablation}
\vspace{-1em}
\end{table}

\begin{table}
\centering
\setlength{\tabcolsep}{4pt}
\resizebox{\columnwidth}{!}{%
\begin{small}
\begin{tabular}{lccccccc}
\toprule
$\lambda$ & DietNeRF & NeRFfacto & APENeRF & FreeNeRF & \makecell{FreeNeRF\\+ SPE} & \makecell{FreeNeRF\\+ GAN}  & \makecell{FreeNeRF\\+ SPE + GAN}\\\midrule
$1$  &0.968  &0.986  &0.982  &0.981 &0.988  &0.982 &0.990\\
$2$  &0.930  &0.972  &0.971  &0.974 &0.975  &0.981 &0.978\\
$3$  &0.833  &0.858  &0.861  &0.877 &0.903  &0.882 & 0.927\\
$4$  &0.590  & 0.682  &0.694  & 0.779 &0.796  &0.760 & 0.894\\
\bottomrule
\end{tabular}%
    \end{small}
    }
\setlength{\tabcolsep}{6pt}
\caption{WPDR with different levels $\lambda$.}
\label{table:power-ratio}
\vspace{-1em}
\end{table}

\begin{figure}[t]
    \centering
    \includegraphics[width=0.45\textwidth]{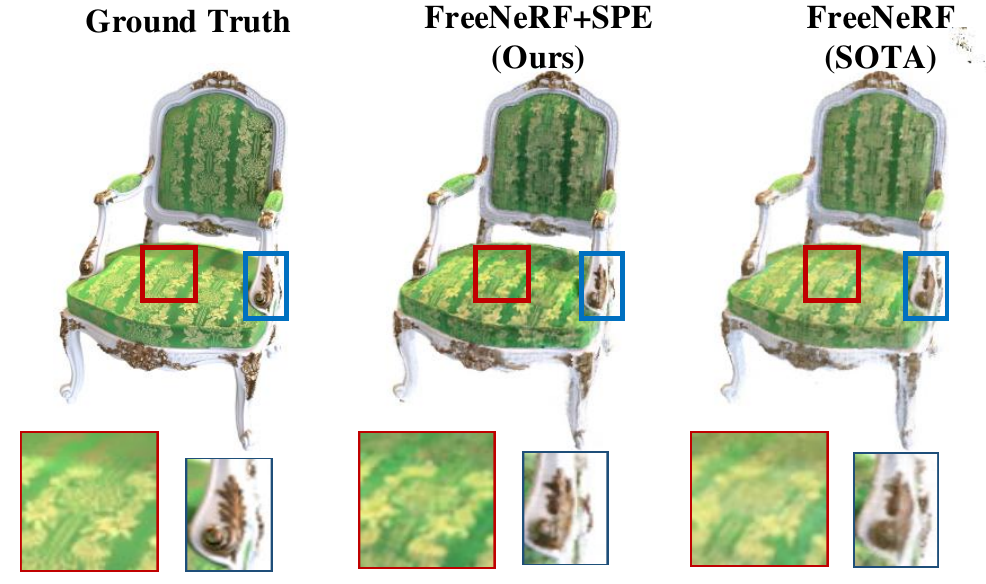} 
    \vspace{-1em}
    \caption{Visual enhancement of Blender Chair with SPE. Our method yields clearer patterns compared to FreeNeRF.}
    \label{fig:chair-overview}
    \vspace{-1em}
\end{figure}


\begin{figure}[t]
    \centering
    \includegraphics[width=0.4\textwidth]{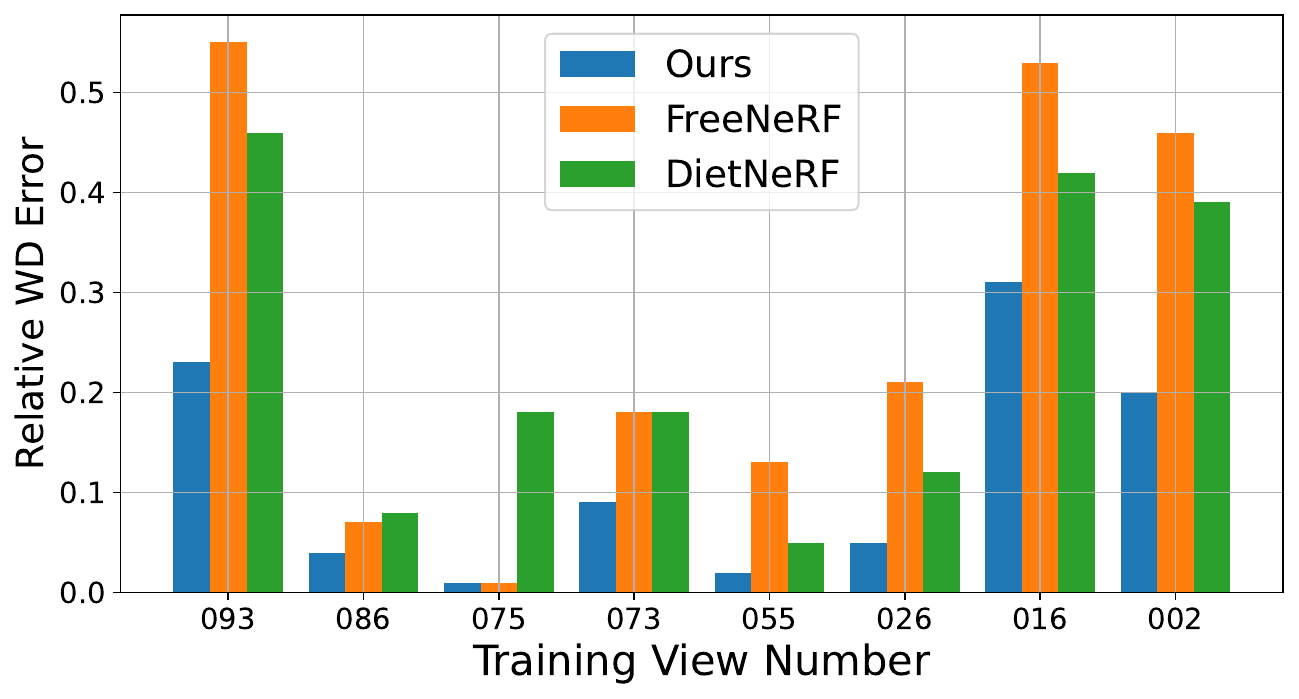} 
    \vspace{-1em}
    \caption{RWDE with different training views.}
    \label{fig:RWDE}
    \vspace{-1em}
\end{figure}

\begin{figure}[t]
    \centering
    \includegraphics[width=0.43\textwidth]{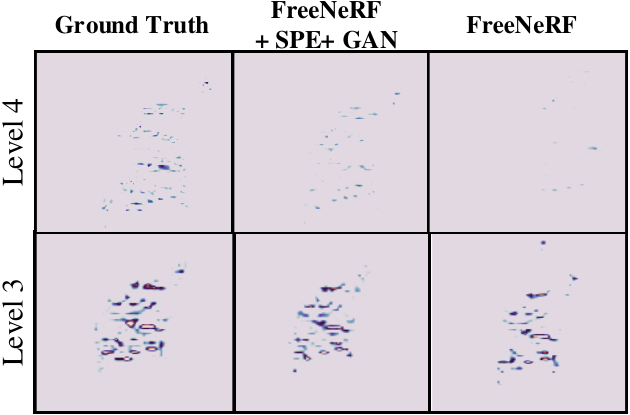} 
    \vspace{-1em}
    \caption{Wavelet decomposition results.}
    \label{fig:waveletdecomp}
    \vspace{-1em}
\end{figure}

We present a performance comparison of various NeRF models and encoding techniques in Table \ref{table:performance-comparison-models}. The table includes seven existing NeRF models categorized by four positional encoding techniques: (non-adaptive) PE (Equation \ref{eq:nerf-pe}), hash encoding (Equation \ref{eq:hash-encoding}), APE (Equation \ref{eq:ape}), and SPE (Equation \ref{eq:spe}). It is evident that our method, SPE, surpasses all other positional encoding techniques in the NeRF task for both PSNR and SSIM metrics. Notably, SPE achieves approximately a 3\% improvement in PSNR and a 1.7\% increase in SSIM compared to the state-of-the-art FreeNeRF \citep{yang2023freenerf} method. This leads to a significant visual enhancement, as illustrated in Figures \ref{fig:chair-overview}. Furthermore, we discovered that GANs \citep{goodfellow2020generative, gulrajani2017improved} can further enhance the performance of PE-based NeRF models.

An ablation study on FreeNeRF, incorporating both SPE and GAN enhancements, is presented in Table \ref{table:ablation}. In this study, we use the SOTA NeRF, FreeNeRF \citep{yang2023freenerf}, as the main baseline. We observe that GANs contribute to performance improvements for both the PE-based (vanilla) FreeNeRF and our SPE-enhanced FreeNeRF, with notably higher accuracy in terms of PSNR and SSIM when combined with SPE. A primary reason is that SPE can provide a richer frequency band (shown in Figure ~\ref{fig:learnedfeature}), where GANs can be exploited with greater flexibility. This can be further reflected in Figure~\ref{fig:RWDE} where GANs can effectively reduce the RWDE which reflects the extent of pixel distortion.

Table \ref{table:power-ratio} presents the WDPR across levels 1-4. The table illustrates that SPE surpasses all other NeRF models across various positional encoding techniques, particularly at higher levels of wavelet decomposition, which correspond to higher frequency features. Notably, at $\lambda = 3$, SPE achieves a 3\% improvement over the state-of-the-art FreeNeRF. Moreover, while integrating GAN with FreeNeRF and PE reduces the power ratio from 0.779 to 0.760, employing GAN with SPE elevates the performance from 0.796 to 0.894, marking a 12.3\% enhancement. This reinforces our assertion that the broader frequency bands facilitated by SPE enhance the efficacy of GANs in training. Such improvements lead to noticeable visual distinctions, as depicted in Figure \ref{fig:waveletdecomp}.

\subsection{Text-to-Speech}\label{sec:speech-synthesis}



\begin{figure}[t]
    \centering
    \includegraphics[width=0.42\textwidth]{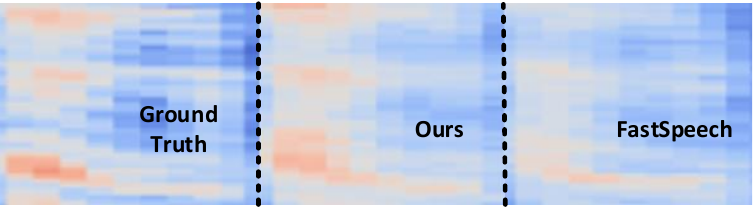}\\
    High Frequency Details\\
    \includegraphics[width=0.42\textwidth]{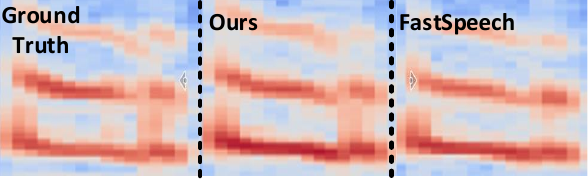}\\
    Low Frequency Details\\
    \caption{Speech spectrum details with red regions indicating signal power. Our method achieves better alignment of these red regions with the ground truth compared to vanilla FastSpeech.\looseness=-1}\label{fig:details}
    \vspace{-1em}
\end{figure}

\begin{table}[t]
\centering
\setlength{\tabcolsep}{4pt}
\begin{sc}
\begin{small}
    \begin{tabular}{lrrrrr}
        \toprule
        Method & Case 1 & Case 2 & Case 3 & Case 4 & Case 5 \\ \midrule
        SPE & \textbf{104.4} & \textbf{97.7} & \textbf{103.2} & \textbf{103.9} & \textbf{97.6} \\ 
        PE & 103.4 & 96.3  & 101.6 & 102.2 &  96.4 \\ 
        FS & 100.7  & 94.4 & 100.2 & 99.3 & 94.7 \\ \hline 
        SPE & \textbf{0.223}  & \textbf{0.223}  & \textbf{0.239} & \textbf{0.240} & \textbf{0.262} \\ 
        PE & 0.210 & 0.209 & 0.228 & 0.229  & 0.236 \\
        FS & 0.191 & 0.191 & 0.205 & 0.204  & 0.225\\ \bottomrule
    \end{tabular}
    \end{small}
\end{sc}
\setlength{\tabcolsep}{6pt}
\caption{Performance comparison of position encoding methods in speech sytnehsis for different cases. Top half is PSNR and bottom half is SSIM, both the larger the better.}
\label{table:fastspeech-psnr-ssim}
\vspace{-1em}
\end{table}

We then evaluate SPE in speech synthesis tasks. FastSpeech \citep{ren2019fastspeech} is selected as the base method for speech use case because the linear layers tend to be a bottleneck in FastSpeech (see ``Figure 1 (a), Feed-Forward Transformer'' in \citep{ren2019fastspeech}). For this model there is no official code available, we therefore use the implementation in \citep{xcmyzfastspeech}, which shows negligible difference on performance compared with the official audio samples \citep{officialfastspeech}.

In Table \ref{table:fastspeech-psnr-ssim}, a comparative analysis of three audio synthesis methods is conducted, with PSNR and SSIM as metrics. In the table, ``SPE'' refers to a model integrating SPE with the original FastSpeech. ``PE'' denotes the incorporation of PE (Equation \ref{eq:nerf-pe}) into FastSpeech, and ``FS'' represents the original FastSpeech model. To make a fair comparison, we set $L = 5$ for PE in all methods. The computation of PSNR and SSIM involves contrasting the synthesized audio from three methods against the ground truth (target audio spectrum graph). These metrics are crucial for assessing the fidelity and perceptual quality of synthesized audio, quantitatively measuring each method's approximation to the ground truth.

The table assesses metrics across five cases with high-frequency features, similar to high-frequency details from Figure~\ref{fig:details}. The results lead to two main observations. First, SPE consistently has the best performance in both PSNR and SSIM across all cases. Second, PE demonstrates superior performance over FS in all scenarios. These findings affirm that PE can enhance fidelity for high-frequency feature generation, and the new SPE further amplifies this benefit by adaptively selecting the most relevant features.

Notably, Case 4 exhibits a remarkable improvement of 4.6\% compared to vanilla FastSpeech, highlighting the efficacy of SPE in enhancing signal quality. Similarly, for SSIM values, which measure the structural similarity to the ground truth, a parallel trend is observed. The SPE-augmented FastSpeech achieves SSIM scores ranging from 0.223 to 0.262, outperforming the original FastSpeech model, which scores between 0.191 and 0.225. This enhancement indicates that SPE not only improves the fidelity of the audio signal but also more effectively preserves its structural integrity compared to the baseline model. These findings underscore the capability of SPE to significantly boost performance in speech generation use cases.

\subsection{1D Regression with Neural Tangent Kernel}\label{sec:ntk}

\begin{figure}[t]
    \centering
    \begin{subfigure}[b]{\columnwidth}
        \centering
        \includegraphics[width=\textwidth]{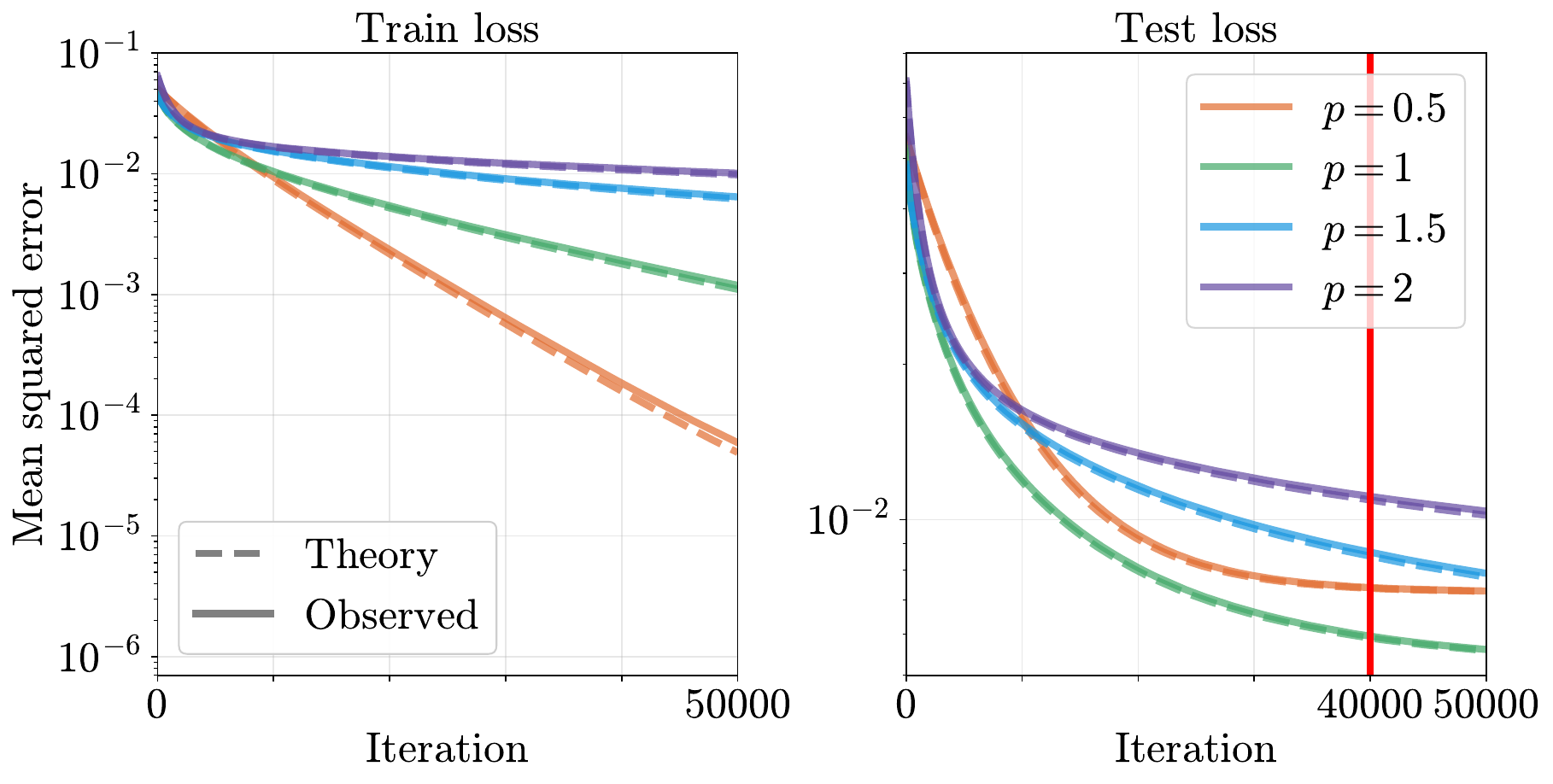}
    \caption{PE}\label{fig:ntk_orig}
    \end{subfigure}
    \begin{subfigure}[b]{\columnwidth}
        \centering
        \includegraphics[width=\textwidth]{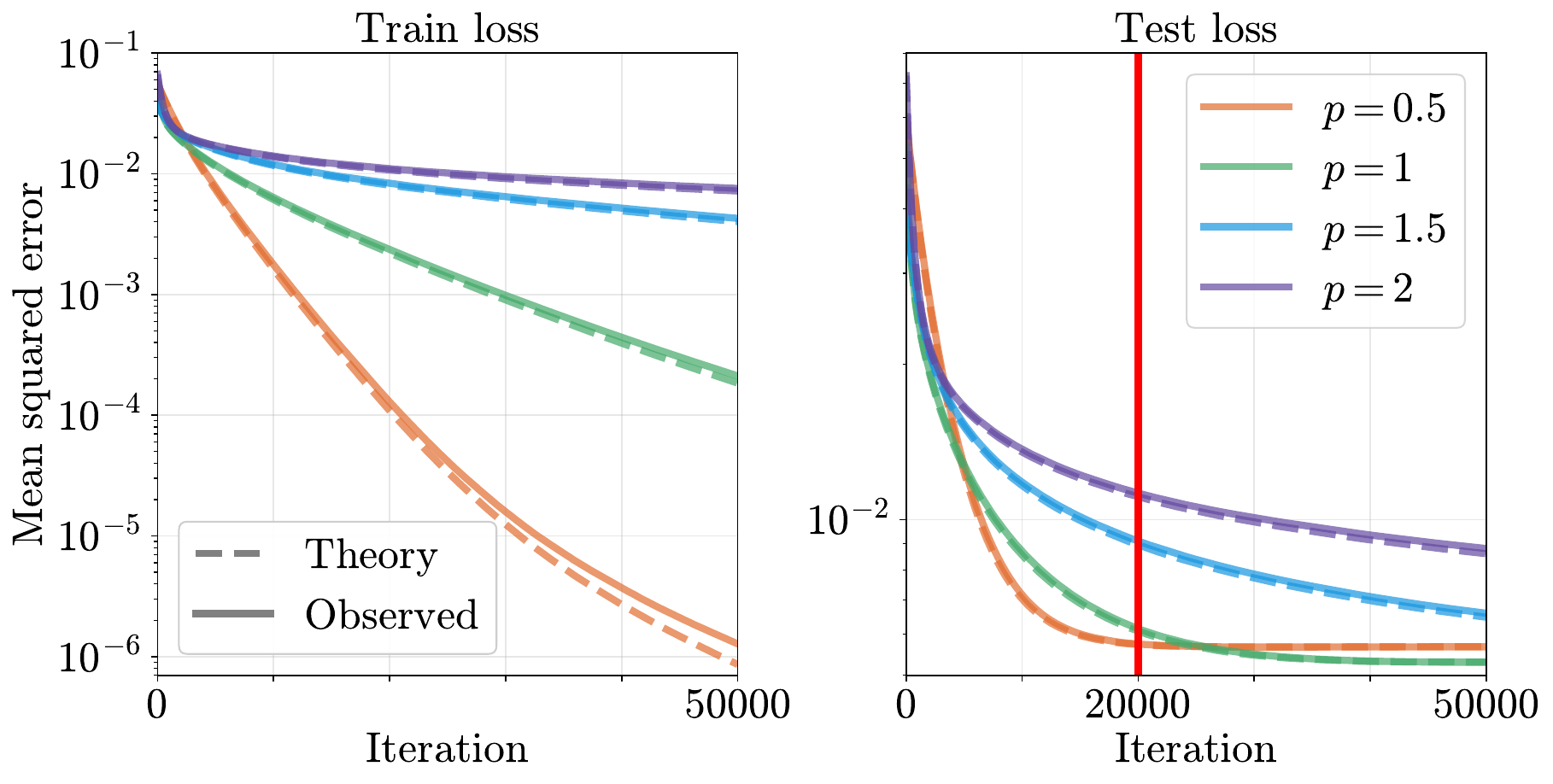}
        \caption{SPE}\label{fig:ntk_our}
    \end{subfigure}
    \caption{Training and test performance of NTK with PE and SPE. Each figure shows the trajectories with different $p$. Dashed lines indicate the trend in theory and solid lines are from experiments. The red vertical line marks the iteration at which PE and SPE achieve equivalent levels of convergence.}\label{fig:ntk}
    \vspace{-1em}
\end{figure}

We finally evaluate SPE on 1D regression with NTK. We adhered to the official implementation of PE-based NTK~\cite {tancik2020fourier} and conducted the experiments accordingly. We modify only by replacing the PE component with SPE. To ensure a fair comparison, we following the same implementaition in \cite{tancik2020fourier}, using Jax \citep{bradbury2018jax} and the same Python library \citep{novak2019neural} to calculate NTK functions automatically.

Figures \ref{fig:ntk} show the NTK regression performance. Here, mappings with smaller $p$ values are associated with NTKs better equipped to capture high-frequency features. We make several observations from this figure. First of all, the training loss plots indicate that, within 50,000 iterations, SPE achieves \textit{~2 $\times$ faster convergence} across all $p$ values compared to PE. Furthermore, the training loss with SPE is notably reduced, being \textit{two orders} of magnitude smaller compared to PE when $p$ equals $0.5$. This significant reduction highlights SPE's enhanced capability in effectively capturing high-frequency features within the training dataset. 

Additionally, in terms of testing loss, while the PE-based NTK regression model requires up to 40,000 iterations to converge, the SPE-based model demonstrates convergence within approximately 20,000 iterations for all evaluated $p$ values with a smaller test loss. 

Finally, the test loss plots for SPE reveal a nuanced trend. While employing NTK regression with PE yields best performance at a $p$ value of 1, as smaller $p$ values tend to cause overfitting, which is discussed in \citet{tancik2020fourier}; SPE demonstrates a marked improvement in this aspect. Specifically, as illustrated in Figure~\ref{fig:ntk}, SPE maintains comparable performance levels even when $p$ is reduced to 0.5. This capacity to effectively handle smaller $p$ values which is intended to capture more high-frequency features substantiates our argument that SPE can adaptively select and converge upon the most pertinent frequency features without the typical drawbacks associated with smaller $p$ values in PE, which is aligned with Theorem \ref{thm:pe-accuracy} in \S\ref{sec:implementation} .
\section{Related Work}\label{relatedworks}
Our study is motivated by the successful application of Fourier features based PE in various tasks for processing visual signals, including images~\citep{stanley2007compositional} and 3D scenes~\citep{mildenhall2021nerf, yang2023freenerf}. Even earlier, there are similar positional encoding methods in natural language processing and time series analysis~\citep{kazemi2019time2vec, vaswani2017attention}. Among these works, \citet{xu2019self} use random Fourier features to approximate stationary kernels with a sinusoidal input mapping and propose techniques to tune the mapping parameters. Then in \citet{tancik2020fourier}, the authors extend Fourier features based PE by comprehending such mapping as a modification of the resulting network's NTK. 

At the same time as \citet{tancik2020fourier}, another work on utilizing sinusoidal activation via SIREN layers to learn from the derivative information~\cite{sitzmann2020implicit} shows another potential of using Fourier representations to help the training process, which is then proved as a structural similar representation as the PE with a trainable parameter~\citep{benbarka2022seeing}.
This provokes us to think about how to use this connection to enable a trainable and adaptive PE. Simultaneously, there are other attempts to learn the Fourier features directly~\citep{gao2023adaptive} or learn a pseudo random encoding of the inputs~\citep{muller2022instant, tancik2023nerfstudio}. However, those works cannot fully address practical concerns, such as why the empirical setting of NeRF should be like in~\citet{mildenhall2021nerf}, for which we see a gap between theory and empirical operation. 

Building on the existing analysis of PE, we introduce SPE to adaptively learn optimal Fourier features, thereby bridging theoretical understanding and practical implementation. This approach not only enhances the effectiveness of PE in capturing high-frequency functions but also broadens the scope of its application across domains.
\section{Conclusions}\label{conclusion}
This paper presents SPE, a new positional encoding method designed to learn adaptive frequency features closely aligned with the true underlying function. Through extensive evaluation across three distinct scenarios -- 1D regression, 2D speech synthesis, and 3D NeRF -- we have demonstrated SPE's superiority over traditional Fourier feature based PE. Our findings indicate that SPE
is effective and efficient at
learning high-frequency functions, underscoring its potential as a versatile tool for a broad set of applications.


\section*{Impact Statement}
This paper presents work whose goal is to advance the field of Deep Learning. There are many potential societal consequences of our work, none which we feel must be specifically highlighted here.
\bibliography{base}

\begin{thebibliography}{30}
\providecommand{\natexlab}[1]{#1}
\providecommand{\url}[1]{\texttt{#1}}
\expandafter\ifx\csname urlstyle\endcsname\relax
  \providecommand{\doi}[1]{doi: #1}\else
  \providecommand{\doi}{doi: \begingroup \urlstyle{rm}\Url}\fi

\bibitem[Arora et~al.(2019)Arora, Du, Hu, Li, and Wang]{arora2019fine}
Arora, S., Du, S., Hu, W., Li, Z., and Wang, R.
\newblock Fine-grained analysis of optimization and generalization for overparameterized two-layer neural networks.
\newblock In \emph{International Conference on Machine Learning}, pp.\  322--332. PMLR, 2019.

\bibitem[Barron et~al.(2021)Barron, Mildenhall, Tancik, Hedman, Martin-Brualla, and Srinivasan]{barron2021mip}
Barron, J.~T., Mildenhall, B., Tancik, M., Hedman, P., Martin-Brualla, R., and Srinivasan, P.~P.
\newblock Mip-nerf: A multiscale representation for anti-aliasing neural radiance fields.
\newblock In \emph{Proceedings of the IEEE/CVF International Conference on Computer Vision}, pp.\  5855--5864, 2021.

\bibitem[Benbarka et~al.(2022)Benbarka, H{\"o}fer, Zell, et~al.]{benbarka2022seeing}
Benbarka, N., H{\"o}fer, T., Zell, A., et~al.
\newblock Seeing implicit neural representations as fourier series.
\newblock In \emph{Proceedings of the IEEE/CVF Winter Conference on Applications of Computer Vision}, pp.\  2041--2050, 2022.

\bibitem[Bradbury et~al.(2018)Bradbury, Frostig, Hawkins, Johnson, Leary, Maclaurin, Necula, Paszke, VanderPlas, Wanderman-Milne, et~al.]{bradbury2018jax}
Bradbury, J., Frostig, R., Hawkins, P., Johnson, M.~J., Leary, C., Maclaurin, D., Necula, G., Paszke, A., VanderPlas, J., Wanderman-Milne, S., et~al.
\newblock Jax: composable transformations of python+ numpy programs.
\newblock 2018.

\bibitem[Gao et~al.(2023)Gao, Dai, and Zhang]{gao2023adaptive}
Gao, Z., Dai, W., and Zhang, Y.
\newblock Adaptive positional encoding for bundle-adjusting neural radiance fields.
\newblock In \emph{Proceedings of the IEEE/CVF International Conference on Computer Vision}, pp.\  3284--3294, 2023.

\bibitem[Goodfellow et~al.(2020)Goodfellow, Pouget-Abadie, Mirza, Xu, Warde-Farley, Ozair, Courville, and Bengio]{goodfellow2020generative}
Goodfellow, I., Pouget-Abadie, J., Mirza, M., Xu, B., Warde-Farley, D., Ozair, S., Courville, A., and Bengio, Y.
\newblock Generative adversarial networks.
\newblock \emph{Communications of the ACM}, 63\penalty0 (11):\penalty0 139--144, 2020.

\bibitem[Gulrajani et~al.(2017)Gulrajani, Ahmed, Arjovsky, Dumoulin, and Courville]{gulrajani2017improved}
Gulrajani, I., Ahmed, F., Arjovsky, M., Dumoulin, V., and Courville, A.~C.
\newblock Improved training of wasserstein gans.
\newblock \emph{Advances in neural information processing systems}, 30, 2017.

\bibitem[Jacot et~al.(2018)Jacot, Gabriel, and Hongler]{jacot2018neural}
Jacot, A., Gabriel, F., and Hongler, C.
\newblock Neural tangent kernel: Convergence and generalization in neural networks.
\newblock \emph{Advances in neural information processing systems}, 31, 2018.

\bibitem[Jain et~al.(2021)Jain, Tancik, and Abbeel]{jain2021putting}
Jain, A., Tancik, M., and Abbeel, P.
\newblock Putting nerf on a diet: Semantically consistent few-shot view synthesis.
\newblock In \emph{Proceedings of the IEEE/CVF International Conference on Computer Vision}, pp.\  5885--5894, 2021.

\bibitem[Kazemi et~al.(2019)Kazemi, Goel, Eghbali, Ramanan, Sahota, Thakur, Wu, Smyth, Poupart, and Brubaker]{kazemi2019time2vec}
Kazemi, S.~M., Goel, R., Eghbali, S., Ramanan, J., Sahota, J., Thakur, S., Wu, S., Smyth, C., Poupart, P., and Brubaker, M.
\newblock Time2vec: Learning a vector representation of time.
\newblock \emph{arXiv preprint arXiv:1907.05321}, 2019.

\bibitem[Lee et~al.(2019)Lee, Xiao, Schoenholz, Bahri, Novak, Sohl-Dickstein, and Pennington]{lee2019wide}
Lee, J., Xiao, L., Schoenholz, S., Bahri, Y., Novak, R., Sohl-Dickstein, J., and Pennington, J.
\newblock Wide neural networks of any depth evolve as linear models under gradient descent.
\newblock \emph{Advances in neural information processing systems}, 32, 2019.

\bibitem[Liu(2020)]{xcmyzfastspeech}
Liu, Z.
\newblock Fastspeech-pytorch.
\newblock \url{https://github.com/xcmyz/FastSpeech}, 2020.

\bibitem[Mescheder et~al.(2019)Mescheder, Oechsle, Niemeyer, Nowozin, and Geiger]{mescheder2019occupancy}
Mescheder, L., Oechsle, M., Niemeyer, M., Nowozin, S., and Geiger, A.
\newblock Occupancy networks: Learning 3d reconstruction in function space.
\newblock In \emph{Proceedings of the IEEE/CVF conference on computer vision and pattern recognition}, pp.\  4460--4470, 2019.

\bibitem[Mildenhall et~al.(2021)Mildenhall, Srinivasan, Tancik, Barron, Ramamoorthi, and Ng]{mildenhall2021nerf}
Mildenhall, B., Srinivasan, P.~P., Tancik, M., Barron, J.~T., Ramamoorthi, R., and Ng, R.
\newblock Nerf: Representing scenes as neural radiance fields for view synthesis.
\newblock \emph{Communications of the ACM}, 65\penalty0 (1):\penalty0 99--106, 2021.

\bibitem[M{\"u}ller et~al.(2019)M{\"u}ller, McWilliams, Rousselle, Gross, and Nov{\'a}k]{muller2019neural}
M{\"u}ller, T., McWilliams, B., Rousselle, F., Gross, M., and Nov{\'a}k, J.
\newblock Neural importance sampling.
\newblock \emph{ACM Transactions on Graphics (ToG)}, 38\penalty0 (5):\penalty0 1--19, 2019.

\bibitem[M{\"u}ller et~al.(2022)M{\"u}ller, Evans, Schied, and Keller]{muller2022instant}
M{\"u}ller, T., Evans, A., Schied, C., and Keller, A.
\newblock Instant neural graphics primitives with a multiresolution hash encoding.
\newblock \emph{ACM Transactions on Graphics (ToG)}, 41\penalty0 (4):\penalty0 1--15, 2022.

\bibitem[Nguyen et~al.(2015)Nguyen, Yosinski, and Clune]{nguyen2015deep}
Nguyen, A., Yosinski, J., and Clune, J.
\newblock Deep neural networks are easily fooled: High confidence predictions for unrecognizable images.
\newblock In \emph{Proceedings of the IEEE conference on computer vision and pattern recognition}, pp.\  427--436, 2015.

\bibitem[Novak et~al.(2019)Novak, Xiao, Hron, Lee, Alemi, Sohl-Dickstein, and Schoenholz]{novak2019neural}
Novak, R., Xiao, L., Hron, J., Lee, J., Alemi, A.~A., Sohl-Dickstein, J., and Schoenholz, S.~S.
\newblock Neural tangents: Fast and easy infinite neural networks in python.
\newblock \emph{arXiv preprint arXiv:1912.02803}, 2019.

\bibitem[Rahimi \& Recht(2007)Rahimi and Recht]{rahimi2007random}
Rahimi, A. and Recht, B.
\newblock Random features for large-scale kernel machines.
\newblock \emph{Advances in neural information processing systems}, 20, 2007.

\bibitem[Ren et~al.(2019{\natexlab{a}})Ren, Ruan, Tan, Qin, Zhao, Zhao, and Liu]{officialfastspeech}
Ren, Y., Ruan, Y., Tan, X., Qin, T., Zhao, S., Zhao, Z., and Liu, T.-Y.
\newblock Fastspeech: Fast, robust and controllable text to speech (audio samples).
\newblock \url{https://github.com/xcmyz/FastSpeech}, 2019{\natexlab{a}}.

\bibitem[Ren et~al.(2019{\natexlab{b}})Ren, Ruan, Tan, Qin, Zhao, Zhao, and Liu]{ren2019fastspeech}
Ren, Y., Ruan, Y., Tan, X., Qin, T., Zhao, S., Zhao, Z., and Liu, T.-Y.
\newblock Fastspeech: Fast, robust and controllable text to speech.
\newblock \emph{Advances in neural information processing systems}, 32, 2019{\natexlab{b}}.

\bibitem[Ren et~al.(2020)Ren, Hu, Tan, Qin, Zhao, Zhao, and Liu]{ren2020fastspeech}
Ren, Y., Hu, C., Tan, X., Qin, T., Zhao, S., Zhao, Z., and Liu, T.-Y.
\newblock Fastspeech 2: Fast and high-quality end-to-end text to speech.
\newblock In \emph{International Conference on Learning Representations}, 2020.

\bibitem[Sitzmann et~al.(2019)Sitzmann, Thies, Heide, Nie{\ss}ner, Wetzstein, and Zollhofer]{sitzmann2019deepvoxels}
Sitzmann, V., Thies, J., Heide, F., Nie{\ss}ner, M., Wetzstein, G., and Zollhofer, M.
\newblock Deepvoxels: Learning persistent 3d feature embeddings.
\newblock In \emph{Proceedings of the IEEE/CVF Conference on Computer Vision and Pattern Recognition}, pp.\  2437--2446, 2019.

\bibitem[Sitzmann et~al.(2020)Sitzmann, Martel, Bergman, Lindell, and Wetzstein]{sitzmann2020implicit}
Sitzmann, V., Martel, J., Bergman, A., Lindell, D., and Wetzstein, G.
\newblock Implicit neural representations with periodic activation functions.
\newblock \emph{Advances in neural information processing systems}, 33:\penalty0 7462--7473, 2020.

\bibitem[Stanley(2007)]{stanley2007compositional}
Stanley, K.~O.
\newblock Compositional pattern producing networks: A novel abstraction of development.
\newblock \emph{Genetic programming and evolvable machines}, 8:\penalty0 131--162, 2007.

\bibitem[Tancik et~al.(2020)Tancik, Srinivasan, Mildenhall, Fridovich-Keil, Raghavan, Singhal, Ramamoorthi, Barron, and Ng]{tancik2020fourier}
Tancik, M., Srinivasan, P., Mildenhall, B., Fridovich-Keil, S., Raghavan, N., Singhal, U., Ramamoorthi, R., Barron, J., and Ng, R.
\newblock Fourier features let networks learn high frequency functions in low dimensional domains.
\newblock \emph{Advances in Neural Information Processing Systems}, 33:\penalty0 7537--7547, 2020.

\bibitem[Tancik et~al.(2023)Tancik, Weber, Ng, Li, Yi, Wang, Kristoffersen, Austin, Salahi, Ahuja, et~al.]{tancik2023nerfstudio}
Tancik, M., Weber, E., Ng, E., Li, R., Yi, B., Wang, T., Kristoffersen, A., Austin, J., Salahi, K., Ahuja, A., et~al.
\newblock Nerfstudio: A modular framework for neural radiance field development.
\newblock In \emph{ACM SIGGRAPH 2023 Conference Proceedings}, pp.\  1--12, 2023.

\bibitem[Vaswani et~al.(2017)Vaswani, Shazeer, Parmar, Uszkoreit, Jones, Gomez, Kaiser, and Polosukhin]{vaswani2017attention}
Vaswani, A., Shazeer, N., Parmar, N., Uszkoreit, J., Jones, L., Gomez, A.~N., Kaiser, {\L}., and Polosukhin, I.
\newblock Attention is all you need.
\newblock \emph{Advances in neural information processing systems}, 30, 2017.

\bibitem[Xu et~al.(2019)Xu, Ruan, Korpeoglu, Kumar, and Achan]{xu2019self}
Xu, D., Ruan, C., Korpeoglu, E., Kumar, S., and Achan, K.
\newblock Self-attention with functional time representation learning.
\newblock \emph{Advances in neural information processing systems}, 32, 2019.

\bibitem[Yang et~al.(2023)Yang, Pavone, and Wang]{yang2023freenerf}
Yang, J., Pavone, M., and Wang, Y.
\newblock Freenerf: Improving few-shot neural rendering with free frequency regularization.
\newblock In \emph{Proceedings of the IEEE/CVF Conference on Computer Vision and Pattern Recognition}, pp.\  8254--8263, 2023.

\end{thebibliography}
\bibliographystyle{icml2024}

\newpage
\appendix
\onecolumn

\section{Appendix}\label{appendix}

\subsection{Further Analysis of SPE}\label{appendix:SPE}

\subsubsection{Reasons for using sinusoidal activation}\label{otheract}
We provide an extended discussion of why Sinusoidal activation turns out to be a highly effective option for SPE. If we denote the activation function in SPE with $\mathbf{A}(\cdot)$, to effectively approximate a sinusoidal wave, the condition can be formulated as:

\begin{equation}\label{approxtarget2}
    \exists \mathbf{I}(t), \mathbf{S}(t) \to \mathbf{A}(\boldsymbol{\omega}\sin(t))\cdot\mathbf{I}(t) + \mathbf{A}(\boldsymbol{\omega}\cos(t))\cdot \mathbf{S}(t) = \sin(\boldsymbol{\omega}\cdot t) = \sin(\boldsymbol{\omega}\cdot 2^{l}\pi\mathbf{x}),
\end{equation}

the $\boldsymbol{\omega}$ can be any element in $\mathbf{W}_\text{SPE} = [\boldsymbol{\omega}_1, \dots, \boldsymbol{\omega}_{2L}]$, and $l\in\{0,\dots, L-1\}$. Suppose $\mathbf{A}(\cdot)$ is a continuous, differentiable, and periodic function, $\mathbf{A}(\mathbf{x}) = \mathbf{A}(\mathbf{x}+2\pi)$, we have
\begin{equation}
    \lim_{t\rightarrow n\pi} \mathbf{A}(\boldsymbol{\omega}\sin(t))\cdot\mathbf{I}(t) + \mathbf{A}(\boldsymbol{\omega}\cos(t))\cdot \mathbf{S}(t) = \mathbf{A}(\boldsymbol{\omega}\cdot |t-n\pi|)\cdot\mathbf{I}(t) + \mathbf{A}(\boldsymbol{\omega}\cos(t))\cdot \mathbf{S}(t)
\end{equation}

Considering $\cos{n\pi} \in \{-1, 1\}$ in periodic, to make it works, we now investigate other common periodic activation.

Suppose $\mathbf{A}(\cdot)$ is periodic linear, we need to have
\begin{equation}
    \mathbf{f}(t, \boldsymbol{\omega}) = \alpha_0\cdot(\boldsymbol{\omega}\sin(t)\bmod 2\pi)\cdot\mathbf{I}(t) + \alpha_1\cdot(\boldsymbol{\omega}\cos(t)\bmod 2\pi)\cdot \mathbf{S}(t),
\end{equation}

where $\mathbf{A}(\cdot)$ is a local linear function such as periodic linear. $\alpha_0$ and $\alpha_1$ are constant value varies between different activation and the range of input. To make Equation~\ref{approxtarget2} work, $\mathbf{I}(\cdot)$ and $\mathbf{S}(\cdot)$ cannot be $\boldsymbol{\omega}$ agnostic, for instance, we need
\begin{equation}
    \lim_{\boldsymbol{\omega}\sin(t)\rightarrow 2n\pi} \mathbf{f}(t, \boldsymbol{\omega}) = \alpha_1\cdot(\boldsymbol{\omega}\cos(t)\bmod 2\pi)\cdot \mathbf{S}(t) = \sin(\boldsymbol{\omega}\cdot t),
\end{equation}
we also have, when $\boldsymbol{\omega}\sin(t) = 2n\pi$, because $\boldsymbol{\omega}^2\sin(t)^2 + \boldsymbol{\omega}^2\cos(t)^2 = \boldsymbol{\omega}^2$
\begin{equation}
    \boldsymbol{\omega}\cos(t) = \sqrt{\boldsymbol{\omega}^2-(2n\pi)^2}
\end{equation}
Then we get a $\boldsymbol{\omega}$ related form of $\mathbf{S}(t)$
\begin{equation}\label{eqst}
    \mathbf{S}(t) = \frac{\sin(\boldsymbol{\omega}\cdot t)}{\sqrt{\boldsymbol{\omega}^2-(2n\pi)^2} \bmod 2\pi}
\end{equation}
Therefore, we prove that if the periodic activation follows the form in Figure~\ref{fig:sawtooth}, to represent a mono-frequency feature, it requires the following layers to perform the high-frequency behaviour as well. Because of the analysis in \citep{tancik2020fourier} and \S\ref{ntk}, the following layer can hardly learn this function. As a result, this form of activation cannot efficiently represent a single frequency feature and tend to introduce artifacts on the final output.

\begin{figure}[t]
\centering
\includegraphics[width=0.45\textwidth]{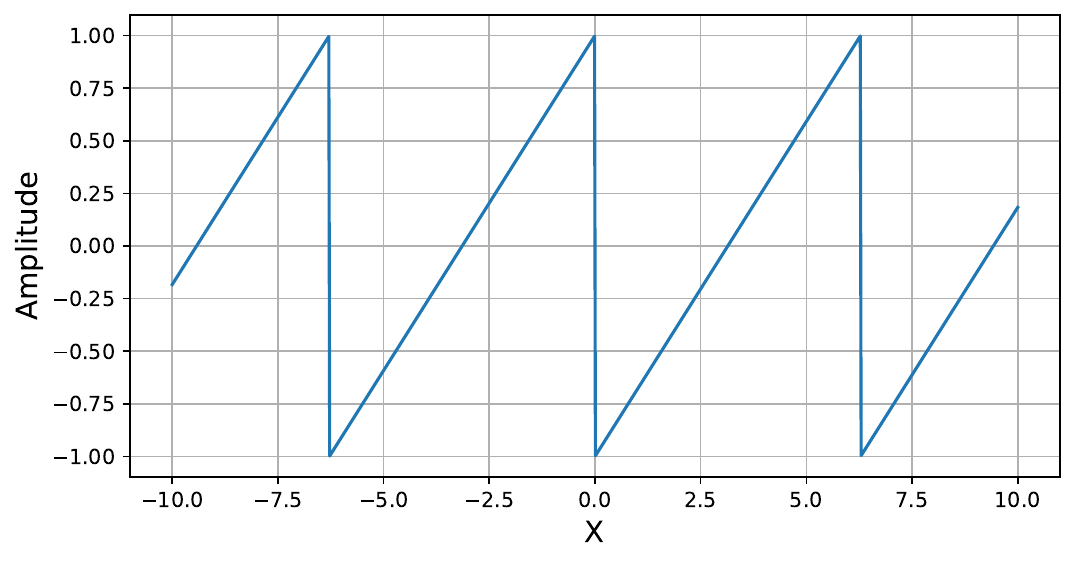} 
\vspace{-1em}
\caption{Saw-tooth Activation is periodic linear and cannot represent sinusoidal wave appropriately.}
\label{fig:sawtooth}
\end{figure}

We can further extend the discussion to any periodic activation function with significant linear behaviour in part of their periodicity, simply by taking out the corresponding period and then we will get $\boldsymbol{\omega}$ related representation of the following layers. \emph{Hence not all periodic activation has the same effectiveness on representing frequency features, if a function has significant linear mapping, then it cannot be a perfect approximation of frequency features as in PE}. 

Another advantage of Sinusoidal activation is the theoretical effectiveness of approximating arbitrary waveform, including the Saw-tooth activation, if desired. According to Fourier transformation, we know that any continuous periodic series can be represented by a linear combination of Sinusoidal waves. Suppose the $L$ is large enough, the other continuous periodic activation is a special case of SPE.

Nevertheless, the Sinusoidal activation used in this paper is not the only solution to make Equation~\ref{approxtarget2} work. Theoretically, for example, it can have low dimension input $\mathbf{x}$ related distortion to the Sinusoidal activation, and then the following layer would have a chance to fix such distortion according to the low dimension input. We choose to use the Sinusoidal activation because it makes the function need to be learned by the following layers to represent a frequency feature simple enough as Equation~\ref{IS0} and Equation~\ref{IS1}.

\subsubsection{Approximation ability}\label{spelearn}
We provide an extended proof which shows SPE can learn an approximation of arbitrary frequency features effectively, \emph{i.e.}, besides SPE, the rest part of MLP is frequency feature agnostic. Considering Equation~\ref{eq:nerf-pe}, we want to prove (let $t = 2^{l}\pi\mathbf{x}, l\in\{0,\dots, L-1\}$)
\begin{equation}\label{approxtarget}
    \exists \mathbf{I}(t), \mathbf{S}(t) \to \sin(\boldsymbol{\omega}\sin(t))\cdot\mathbf{I}(t) + \sin(\boldsymbol{\omega}\cos(t))\cdot \mathbf{S}(t) = \sin(\boldsymbol{\omega}\cdot t) = \sin(\boldsymbol{\omega}
    \cdot 2^{l}\pi\mathbf{x}),
\end{equation}
the $\boldsymbol{\omega}$ can be any element in $\mathbf{W}_\text{SPE} = [\boldsymbol{\omega}_1, \dots, \boldsymbol{\omega}_{2L}]$,
so that it can represent the frequency features without requiring the rest layers to take in the new features but only condition on the input $t$. Assume frequency feature $\boldsymbol{\omega} \in \mathbb{N}^{+}$, when $t\rightarrow n\pi$, if we have $\mathbf{I}(t) = 1$ and $\mathbf{S}(t) = 0$, then the Equation~\ref{approxtarget} can be simplified as following
\begin{equation}\label{approxnpi}
    \lim_{t \rightarrow  n\pi} \sin(\boldsymbol{\omega}\sin(t))
    = \sin(\boldsymbol{\omega}\cdot t)
    = \sin(\boldsymbol{\omega}\cdot 2^{l}\pi\mathbf{x})
\end{equation}
Therefore, whenever $t \rightarrow  n\pi$, $\boldsymbol{\omega}$ is a good approximation of the actual frequency feature. 

Similarly, when $t\rightarrow (n+\frac{1}{2})\pi$, if we have $\mathbf{I}(t) = 0$ and $\mathbf{S}(t) = 1$, Equation~\ref{approxtarget} can be simplified as following
\begin{equation}\label{approxn2pi}
    \lim_{t \rightarrow  n\pi} \sin(\boldsymbol{\omega}\cos(t))
    = \sin(\boldsymbol{\omega}\cdot t)
    = \sin(\boldsymbol{\omega}\cdot 2^{l}\pi\mathbf{x})
\end{equation}

Suppose we follow the conventional PE as the initial part of SPE, the requirement of approximation on $t$ can be transmitted into the following form:

\begin{equation}\label{IS0}
    \text{If:}~2^{l}\pi \mathbf{x}\rightarrow (n+\frac{1}{2})\pi,~\text{Then:}~\mathbf{I}(t) = 1, \mathbf{S}(t) = 0, 
\end{equation}

\begin{equation}\label{IS1}
    \text{If:}~2^{l}\pi \mathbf{x}\rightarrow n\pi,~\text{Then:}~\mathbf{I}(t) = 0, \mathbf{S}(t) = 1
\end{equation}

Therefore, as long as $\exists L\in\mathbb{N}^+ \rightarrow 2^{l}\pi \mathbf{x} \simeq \frac{n\pi}{2}$, the corresponding input value of $\mathbf{x}$ will be approximately represented by frequency $\boldsymbol{\omega}$. To evaluate the effectiveness, we take the highest frequency $l=2^{L-1}$, then we can compute the possible non-zero value $x_i$ in $\mathbf{x}$ can always be represented in the following form
\begin{equation}\label{theox}
    x_i = \frac{n}{2^L} + \epsilon, n \in\mathbb{N}^+, L \in\mathbb{N}^+, \epsilon \rightarrow 0
\end{equation}
We now prove any real $x_i$ can be written in the form of Equation \ref{theox} with a marginal difference. Given arbitrary value of $x_i\in \mathbb{R}$, let $n$ is the nearest integer to $2^Lx_i$, which has $|n - 2^Lx_i| \leq \frac{1}{2}$, then 
\begin{equation}
    \epsilon = x_i -\frac{n}{2^L} < x_i -\frac{2^Lx_i-1}{2^L} = \frac{1}{2^L}
\end{equation}
Therefore the length of PE $L$ determines the resolution of frequency approximation. For the common configuration $L>10$, the error $\epsilon$ is an ignorable small value. Therefore, if write in the vector form of $\mathbf{x}$, we have

\begin{equation}
    \sin(\boldsymbol{\omega}\mathrm{PE}(\mathbf{x})) = \PE(\boldsymbol{\omega}(\mathbf{x}+\frac{1}{2^L}))
\end{equation}

The value $L$ is larger the better, and it should match network size. $L \propto N_{para}$.

We notice that $\mathbf{I}(\cdot)$ and $\mathbf{s}(\cdot)$ is easy to learn because it is just a linear classification to the value of $\mathbf{x}$ itself and totally $\boldsymbol{\omega}$ agnostic. Hence, the representation of SPE on the following layers is almost the same as PE with frequency feature $\boldsymbol{\omega}$, where the error bound of approximation is controlled by $L$. We also prove that with SPE, using larger $L$ is encouraged as it comes with more possible frequency features and lower errors.

Keeping those low-frequency features in SPE is fine because MLP does not suffer from learning low-frequency features, and those parts will be anyway trained well.

\subsubsection{Relation between PE and SPE}\label{spehaspe}
We provide extended proof to show PE is a special case of SPE. We discuss a function $\mathbf{F}(\cdot)$ that maps the SPE components into corresponding high-frequency functions, where we have
\begin{equation}
    \mathbf{F}( \mathbf{x}) =  \mathcal{F}( \mathbf{SPE}(\mathbf{x})),
\end{equation}
where $\mathcal{F}(\cdot)$ is the network after SPE, and by default they are FC layers. With SIREN activation, the SPE can be formulated as follows, for $ l\in\{0,\dots, L-1\}$
\begin{equation}
    \mathbf{SPE}(\mathbf{x}) = \sin(\boldsymbol{\omega}_{\text{SPE}}\cdot\sin(2^{l}\mathbf{x}))
\end{equation}
The original PE plus the first FC layer can be formulated as
\begin{equation}
    \mathbf{PE}(\mathbf{x}) = \boldsymbol{\omega}_{\text{PE}}\cdot \sin(2^{l}\mathbf{x})
\end{equation}
When $\mathbf{x}=0$, $\mathbf{PE}(\mathbf{x}) = \mathbf{SPE}(\mathbf{x}) = 0$. Otherwise, the ratio between $\mathbf{SPE}$ and $\mathbf{PE}$ is
\begin{equation}\label{spepe}
    \frac{\mathbf{SPE}(\mathbf{x})}{\mathbf{PE}(\mathbf{x})} = \frac{\sin(\boldsymbol{\omega}_{\text{SPE}}\cdot\sin(2^{l}\mathbf{x}))}{\boldsymbol{\omega}_{\text{PE}}\cdot\sin(2^{l}\mathbf{x})} (\mathbf{x}\neq 0)
\end{equation}
Essentially, when $\boldsymbol{\omega}_{\text{SPE}}\rightarrow 0$,
because $|\sin(2^{l}\mathbf{x})|\leq 1$, we have 
\begin{equation}  \frac{\sin(\boldsymbol{\omega}_{\text{SPE}}\cdot\sin(2^{l}\mathbf{x}))}{\boldsymbol{\omega}_{\text{SPE}}\cdot\sin(2^{l}\mathbf{x})} = 1
\end{equation}
Therefore, SPE learn an approximation of input PE by learning small weight $\boldsymbol{\omega}_{\text{SPE}}$
\begin{equation}
\lim_{\boldsymbol{\omega}_{\text{SPE}}\rightarrow \epsilon} \frac{\mathbf{SPE}(\mathbf{x})}{\mathbf{PE}(\mathbf{x})}\simeq\frac{\boldsymbol{\omega}_{\text{SPE}}\cdot\sin(2^{l}\mathbf{x})}{\boldsymbol{\omega}_{\text{PE}}\cdot\sin(2^{l}\mathbf{x})} = \frac{\boldsymbol{\omega}_{\text{SPE}}}{\boldsymbol{\omega}_{\text{PE}}}
\end{equation}
By using small weights in SPE~\footnote{A simpler example is, $\lim_{x\rightarrow0} \frac{\sin{5x}}{x} = 5$}, it achieve same performance a PE. Considering there are following layers that can easily scale up the value, we conclude that SPE has a similar behaviour as the original PE with one FC layer. 

\subsection{Further results of NeRF experiments}\label{appendix:EMD}

For the training of SPE in FreeNeRF, we include adversarial training, \emph{a.k.a.}, GAN to minimize the Wasserstein distance between the output and target view. We find that when turning to a new view, the change in the distribution of pixels can represent how different the new view to the previous one. 

We follow a trajectory to observe a chair in the Bender dataset. The closer the number is, the closer the viewpoint is. We can see the process of distribution shifting. From Figure~\ref{fig:hist0} to Figure~\ref{fig:hist24}, we see the Wasserstein distance is a good metric to describe how the pattern in the view is different to the others. Good fidelity means the Wasserstein distance between the training set to the test set should be equal to the distance between the training set and the synthetic views. If the synthetic data has a smaller distance, then it indicates the corresponding method cannot make the output conditional on the input by overfitting the existing distribution.

\begin{figure}[t]
\centering
\includegraphics[width=0.75\textwidth]{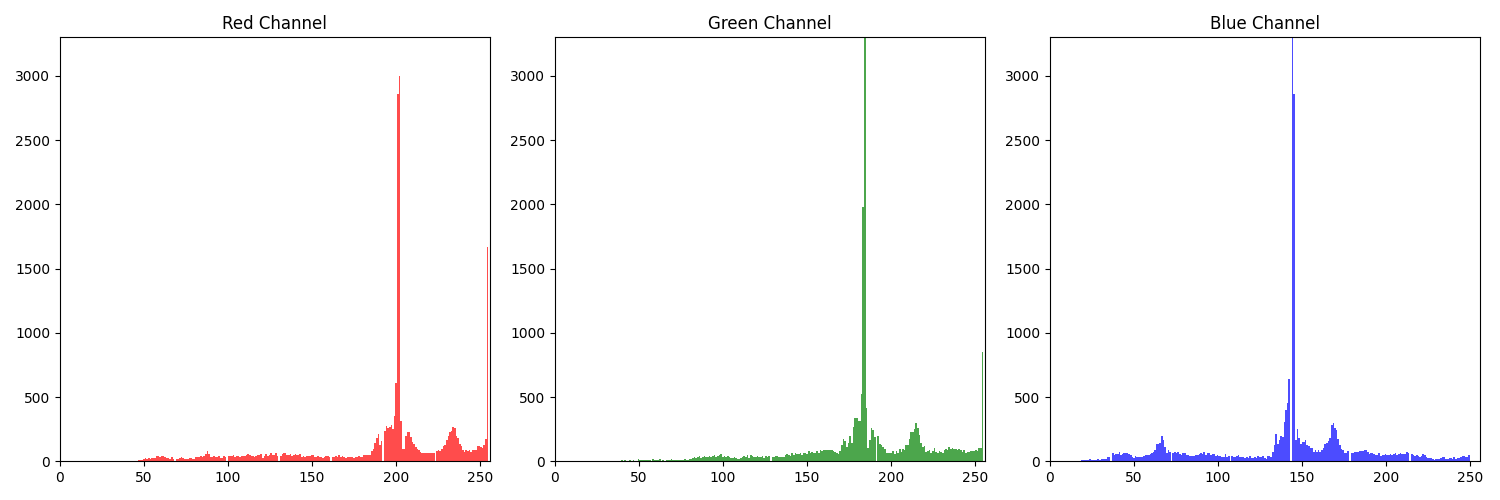} 
\vspace{-1em}
\caption{Histogram of Pixels at view 0.}
\label{fig:hist0}
\end{figure}

\begin{figure}[t]
\centering
\includegraphics[width=0.75\textwidth]{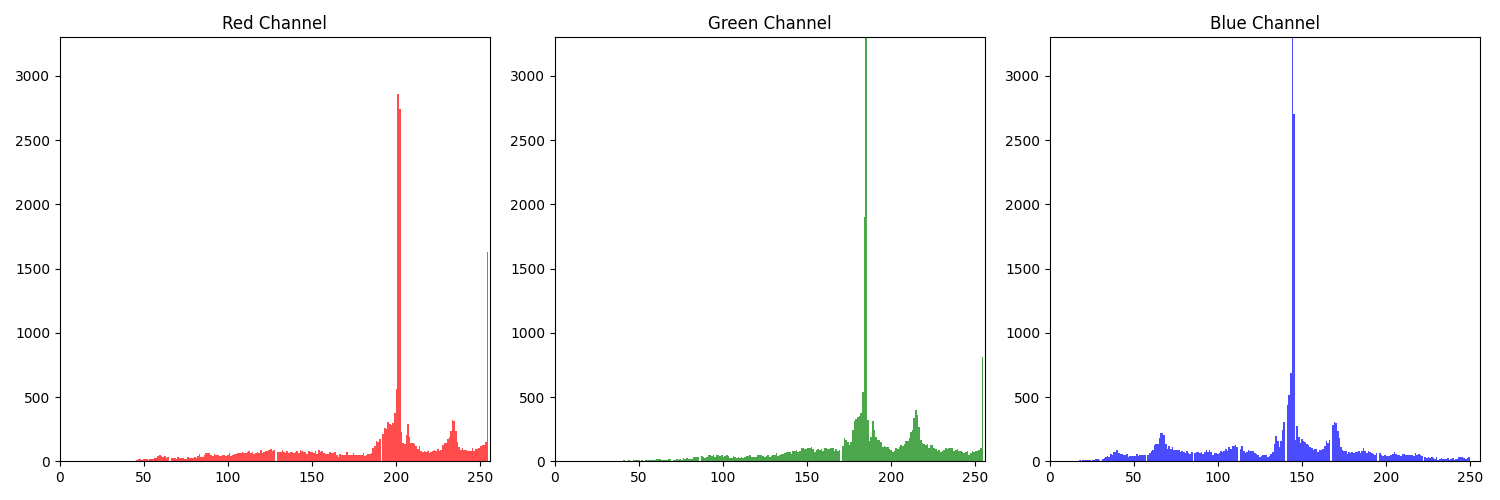} 
\vspace{-1em}
\caption{Histogram of Pixels at view 4.}
\label{fig:hist4}
\end{figure}

\begin{figure}[t]
\centering
\includegraphics[width=0.75\textwidth]{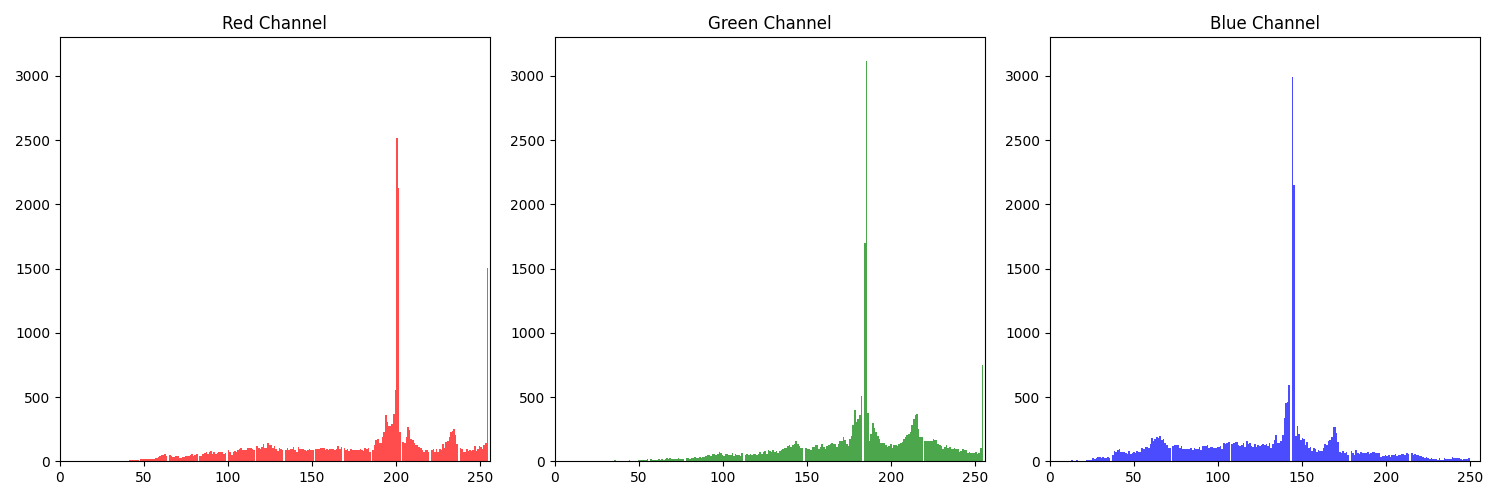} 
\vspace{-1em}
\caption{Histogram of Pixels at view 8.}
\label{fig:hist8}
\end{figure}

\begin{figure}[t]
\centering
\includegraphics[width=0.75\textwidth]{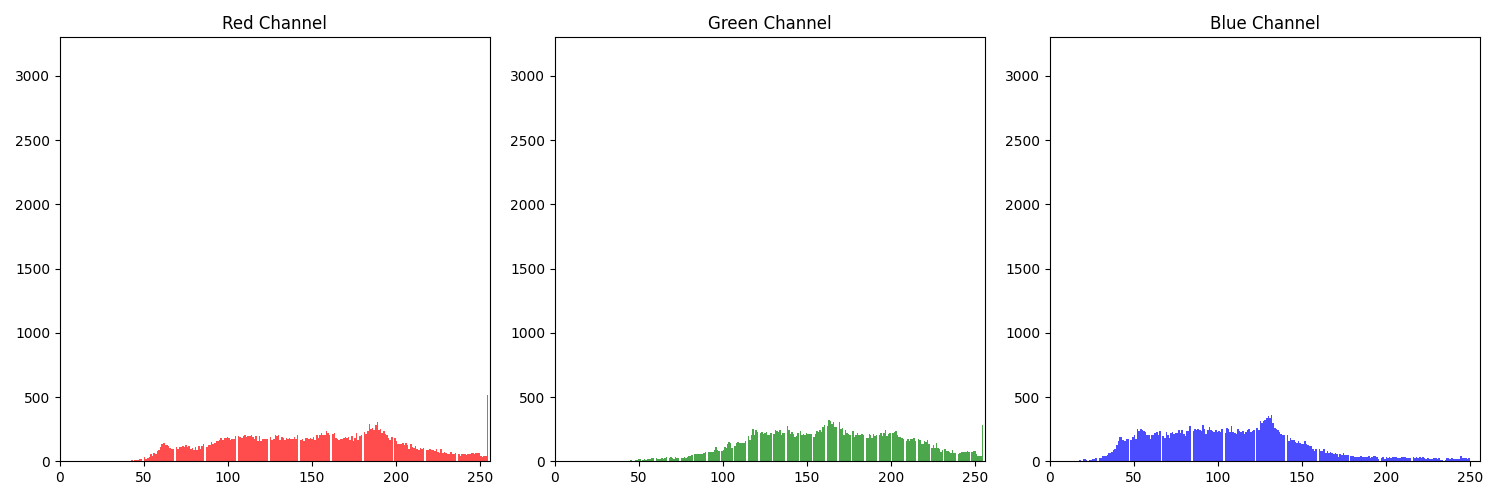} 
\vspace{-1em}
\caption{Histogram of Pixels at view 24.}
\label{fig:hist24}
\end{figure}

This phenomenon is aligned with the fact that pixels will be removed, added, and distorted smoothly when shifting to a new view gradually. However, we find some SOTA methods does not have any gain on learning the right distribution of pixels, even when they have gain on the other metrics such as PSNR and SSIM. Therefore, we make sure the model converges on Wasserstein distance as well to enhance the fidelity.

\subsubsection{Hash encoding and few-view NeRF}
Besides Frequency Features, another encoding method --- Hash Encoding or one-blob encoding, is used in NeRF~\citep{muller2022instant, tancik2020fourier}, and can achieve more accurate results than frequency
encodings in bounded domains at the cost of being single-scale. The one-blob encoding discretizing the input into the bins,  effectively shuts down certain parts of the linear path of the
network, allowing it to specialize the model on various sub-domains of the input. One-blob encoding still uses a stationary kernel without adapting to the actual scene. Also during the discretizing, there is no guarantee it will be done in a way that reflects the high-frequency features. 

Because one-blob encoding uses stationary discretizing and does not learn the frequency feature, the network cannot process the case when the input distribution is far from the training set, where the discretizing leads to a representation that the network cannot generalize. This explains why the one-blob encoding-based method performs much worse on few-view NeRF.

\subsubsection{Training efficiency of SPE}\label{ntk}

Indeed the matrix of $\mathbf{W}_{\text{SPE}}$ is a higher dimension representation than the input sequence. The training of SPE leverages the same mechanism as the original PE, where $\mathbf{W}_{\text{SPE}}$ perform as the first layer of weights as the original PE, therefore the PE part will help the training of $\mathbf{W}_{\text{SPE}}$ as well. 

We conduct the same analysis as \citep{tancik2020fourier, lee2019wide, arora2019fine} with neural tangent kernel (NTK) based neural network approximation. Basically there is not difference when replacing one activation with sinusoidal function, and therefore MLPs with SPE can be trained in the same way just as general MLPs.

\subsubsection{Stationary Kernels}
The significance of stationary kernels, highlighted in \citep{tancik2020fourier}, emphasizes their utility for achieving rotation-invariance and translation-invariance in the Neural Tangent Kernel (NTK). This property is crucial for efficiently training models on diverse data types, including both image (2D/3D) and regression tasks (1D). As demonstrated in Appendix \ref{appendix:SPE}, our SPE method seamlessly aligns with the conventional Positional Encoding (PE) from \citep{tancik2020fourier}, inheriting the advantageous features of stationary kernels, thus offering a robust inductive bias for model training across varied domains.

\section{Supplementary evaluation results} 

\subsection{Configuration details}\label{detailedconfig}
To ensure a fair comparison, we use $L=10$ for the space position input (\emph{i.e.}, the $(x,y,z)$ in Figure~\ref{fig:encoding}), and use $L=4$ for the space position input (\emph{i.e.}, the $\text{dir}$ in Figure~\ref{fig:encoding}), which is the empirical configuration on the Blender dataset. By default, we use the 8-views setup that is aligned with FreeNeRF~\citep{yang2023freenerf} and DietNeRF~\citep{jain2021putting}. We select such a few views setup because if there are too many training views, even the original setup of NeRF can achieve a high fidelity as the new view is similar to the available view. Also when implementing FreeNeRF, we use DietNeRF as the base method of FreeNeRF, which is aligned with the official implementation choice in \citep{yang2023freenerf}.

\subsection{Evaluation of Gaussian Random Fourier Features}\label{eval_grff}
There is another encoding method called Gaussian random Fourier features (GRFF). In GRFF, a pseudo random sequence (still Non-adaptive) is used (\cite{tancik2020fourier} \S 6.1) as Fourier Features. In the SOTA methods, GRFF is not used for two reasons.

\begin{enumerate}
\item Gain is not promising. GRFF assumes Fourier features follow a Gaussian distribution. However, the actual distribution may differ (e.g., Figure~\ref{fig:learnedfeature}), making a Gaussian assumption not universally effective. GRFF's best-case scenario achieves only average performance across known cases, not optimal for specific instances. The FreeNeRF (see Table~\ref{table:GRFF_ablation}) and 1D regression results (see link below) demonstrate that GRFF doesn't significantly outperform PE.

\item Significant computation overhead. To align with the actual distribution of Fourier features in GRFF, an exhaustive search of distribution parameters is required. For instance, the std $\sigma$ for each task and dataset. Given the computational demands, conducting such extensive sweeps for every NeRF task is impractical.
\end{enumerate}

\begin{table}
\centering
\setlength{\tabcolsep}{4pt}
\begin{small}
    \begin{tabular}{lccccc}
    \toprule
    Metric & FreeNeRF  & w/ GRFF ($\sigma=1995$)& w/ SPE  & w/ GAN  & w/ SPE + GAN  \\
    \midrule 
    PSNR $\uparrow$ & 24.259 & 24.382 & 24.951 & 24.531 & 25.202 \\
    SSIM $\uparrow$ & 0.883 &0.886 & 0.898 & 0.889 & 0.910 \\
    \bottomrule
    \end{tabular}%
    \end{small}
\setlength{\tabcolsep}{6pt}
\caption{Evaluation of Gaussian random Fourier features.}
\label{table:GRFF_ablation}
\vspace{-1em}
\end{table}

\subsection{Evaluation of Other Activation Functions with PE}\label{eval_other_func}
The ablation of the choice of a periodic function (based on FreeNeRF) is shown in Table~\ref{table:act_func_ablation}:

\begin{table}
\centering
\setlength{\tabcolsep}{4pt}
\begin{small}
    \begin{tabular}{lcccc}
    \toprule
    Metric & w/ReLU  & w/ Sinusoidal& w/ Sawtooth  & w/ PReLU \\
    \midrule 
    PSNR $\uparrow$ &24.259 &24.951 &23.975 &24.812 \\
    SSIM $\uparrow$ &0.833 &0.898 &0.821 &0.889 \\
    \bottomrule
    \end{tabular}%
    \end{small}
\setlength{\tabcolsep}{6pt}
\caption{Evaluation of Different Activation Functions.}
\label{table:act_func_ablation}
\vspace{-1em}
\end{table}

In this Table~\ref{table:act_func_ablation}, we consider four cases:
\begin{enumerate}
\item ReLU. Basically the original setting of PE.
\item Sinusoidal. The setting of SPE.
\item Saw-tooth functions $f(x)=x-\lfloor{x}\rfloor$ has a bit worse performance to ReLU, perhaps due to the training issue of sawtooth activation.
\item Periodic ReLU (PReLU), $f(x)=\max(0,x)+\sin(x)$. shows a minimal decrease in average performance compared to SPE. The reduction might be caused by the ReLU part as it encourages overfitting to incorrect high-frequency components.
\end{enumerate}

We also tried the case without PE inside, where SPE has a form $\sin(\omega x)$. The result is very blurring with much worse high-frequency features. PSNR is 15.183, and SSIM drops to 0.691.

We will also attend to the format and writing issues in the update version.

\end{document}